\documentclass[letterpaper]{article}
\usepackage[margin=1in,dvips]{geometry}
\usepackage{graphicx,psfrag,amsmath,amsthm,amssymb}
\usepackage{natbib}
\usepackage{url}

\newcommand{\A}{\ensuremath{\mathbf{A}}}

\newcommand{\C}{\ensuremath{\mathbf{C}}}
\newcommand{\D}{\ensuremath{\mathbf{D}}}

\newcommand{\G}{\ensuremath{\mathbf{G}}}

\newcommand{\I}{\ensuremath{\mathbf{I}}}

\newcommand{\LL}{\ensuremath{\mathbf{L}}}
\newcommand{\M}{\ensuremath{\mathbf{M}}}

\newcommand{\PP}{\ensuremath{\mathbf{P}}}
\newcommand{\Q}{\ensuremath{\mathbf{Q}}}
\newcommand{\RR}{\ensuremath{\mathbf{R}}}

\newcommand{\U}{\ensuremath{\mathbf{U}}}

\newcommand{\W}{\ensuremath{\mathbf{W}}}

\newcommand{\Y}{\ensuremath{\mathbf{Y}}}
\newcommand{\Z}{\ensuremath{\mathbf{Z}}}

\renewcommand{\b}{\ensuremath{\mathbf{b}}}
\renewcommand{\c}{\ensuremath{\mathbf{c}}}

\newcommand{\e}{\ensuremath{\mathbf{e}}}

\newcommand{\g}{\ensuremath{\mathbf{g}}}
\newcommand{\h}{\ensuremath{\mathbf{h}}}

\newcommand{\bl}{\ensuremath{\mathbf{l}}}

\newcommand{\p}{\ensuremath{\mathbf{p}}}
\newcommand{\q}{\ensuremath{\mathbf{q}}}

\newcommand{\uu}{\ensuremath{\mathbf{u}}}
\newcommand{\vv}{\ensuremath{\mathbf{v}}}
\newcommand{\w}{\ensuremath{\mathbf{w}}}
\newcommand{\x}{\ensuremath{\mathbf{x}}}
\newcommand{\y}{\ensuremath{\mathbf{y}}}
\newcommand{\z}{\ensuremath{\mathbf{z}}}
\newcommand{\0}{\ensuremath{\mathbf{0}}}
\newcommand{\1}{\ensuremath{\mathbf{1}}}

\newcommand{\bmu}{\ensuremath{\boldsymbol{\mu}}}
\newcommand{\bnu}{\ensuremath{\boldsymbol{\nu}}}

\newcommand{\bpi}{\ensuremath{\boldsymbol{\pi}}}

\newcommand{\bPsi}{\ensuremath{\boldsymbol{\Psi}}}

\newcommand{\bbR}{\ensuremath{\mathbb{R}}}

\newcommand{\calL}{\ensuremath{\mathcal{L}}}

\newcommand{\calO}{\ensuremath{\mathcal{O}}}

\newcommand{\norm}[1]{\left\lVert#1\right\rVert}

\newcommand{\caja}[4][1]{{%
    \renewcommand{\arraystretch}{#1}%
    \begin{tabular}[#2]{@{}#3@{}}%
      #4%
    \end{tabular}%
    }}

{%
\begin{list}{#1}{
\vspace{-\topsep}
\vspace{-\partopsep}
\setlength{\itemindent}{0cm}
\setlength{\rightmargin}{0cm}
\setlength{\listparindent}{0cm}
\settowidth{\labelwidth}{#1}
\setlength{\leftmargin}{\labelwidth}
\addtolength{\leftmargin}{\labelsep}
\setlength{\itemsep}{0cm}
}%
}%
{%
\end{list}
\vspace{-\topsep}
\vspace{-\partopsep}
}

{\begin{enumerate}%
}%
{\end{enumerate}}

\newcommand{\Domain}{\operatorname{dom}}
\newcommand{\dom}[1]{\ensuremath{\Domain\left(#1\right)}}

\newcommand{\diagop}{\operatorname{diag}}
\newcommand{\diag}[1]{\ensuremath{\diagop\left(#1\right)}}
\newcommand{\traceop}{\operatorname{tr}}
\newcommand{\trace}[1]{\ensuremath{\traceop\left(#1\right)}}

\newcommand{\vectop}{\operatorname{vec}}
\newcommand{\vect}[1]{\ensuremath{\vectop\left(#1\right)}}

\theoremstyle{plain}
\newtheorem{thm}{Theorem}[section]

\newtheorem*{lemma*}{Lemma}

\newtheorem*{prop*}{Proposition}
\newtheorem{cor}[thm]{Corollary}

\theoremstyle{definition}

\newtheorem*{defn*}{Definition}

\newtheorem*{exmp*}{Example}

\newtheorem*{conj*}{Conjecture}

\theoremstyle{remark}
\newtheorem{rmk}[thm]{Remark}
\newtheorem*{rmk*}{Remark}

\graphicspath{{grf/}}

\bibpunct[, ]{(}{)}{;}{a}{,}{,} 

\title{LASS: a simple assignment model with Laplacian smoothing}
\author{
  Miguel \'A. Carreira-Perpi\~n\'an \hspace{5ex} Weiran Wang \\
  Electrical Engineering and Computer Science, University of California, Merced \\
  {\url{http://eecs.ucmerced.edu}}
}
\date{May 22, 2014}

\AtBeginDvi{}

\begin{document}

\maketitle

\begin{abstract}

  We consider the problem of learning soft assignments of $N$ items to $K$ categories given two sources of information: an item-category similarity matrix, which encourages items to be assigned to categories they are similar to (and to not be assigned to categories they are dissimilar to), and an item-item similarity matrix, which encourages similar items to have similar assignments. We propose a simple quadratic programming model that captures this intuition. We give necessary conditions for its solution to be unique, define an out-of-sample mapping, and derive a simple, effective training algorithm based on the alternating direction method of multipliers. The model predicts reasonable assignments from even a few similarity values, and can be seen as a generalization of semisupervised learning. It is particularly useful when items naturally belong to multiple categories, as for example when annotating documents with keywords or pictures with tags, with partially tagged items, or when the categories have complex interrelations (e.g.\ hierarchical) that are unknown.

\end{abstract}

\section{Introduction}

A major success in machine learning in recent years has been the development of semisupervised learning (SSL) \citep{Chapel_06a}, where we are given labels for only a few of the training points. Many SSL approaches rely on a neighborhood graph constructed on the training data (labeled and unlabeled), typically weighted with similarity values. The Laplacian of this graph is used to construct a quadratic nonnegative function that measures the agreement of possible labelings with the graph structure, and minimizing it given the existing labels has the effect of propagating them over the graph. Laplacian-based formulations are conceptually simple, computationally efficient (since the Laplacian is usually sparse), have a solid foundation in graph theory and linear algebra \citep{Chung97a,DoyleSnell84a}, and most importantly work very well in practice. The graph Laplacian has been widely exploited in machine learning, computer vision and graphics, and other areas: as mentioned, in semisupervised learning, manifold regularization and graph priors \citep{Zhu_03a,Belkin_06a,Zhou_04a} for regression, classification and applications such as supervised image segmentation \citep{Grady06a}, where one solves a Laplacian-based linear system; in spectral clustering \citep{ShiMalik00a}, possibly with constraints \citep{LuCarreir08a}, and spectral dimensionality reduction \citep{BelkinNiyogi03b} and probabilistic spectral dimensionality reduction \citep{CarreirLu07a}, where one uses eigenvectors of the Laplacian; in clustering, manifold denoising and surface smoothing \citep{Carreir06b,WangCarreir10a,Taubin95a}, where one iterates products of the data with the Laplacian; etc.

We concern ourselves with assignment problems in a semisupervised learning setting, where we have $N$ items and $K$ categories and we want to find soft assignments of items to categories given some information. This information often takes the form of partial tags or annotations, e.g.\ for pictures in websites such as Flickr, blog entries, etc. Let us consider a specific example where the items are documents (e.g.\ papers submitted to this conference) and the categories are keywords. Any given paper will likely be associated to a larger or smaller extent with many keywords, but most authors will tag their papers with only a few of them, usually the most distinctive (although, as we know, there may be other reasons). Thus, few papers will be tagged as ``computer science'' or ``machine learning'' because those keywords are perceived as redundant given, say, ``semisupervised learning''. However, considered in a larger context (e.g.\ to include biology papers), such keywords would be valuable. Besides, categories may have various correlations that are unknown to us but that affect the assignments. For example, a hierarchical structure implies that ``machine learning'' belongs to ``computer science'' (although it does to ``applied maths'' to some extent as well). In general, we consider categories as sets having various intersection, inclusion and exclusion relations. Section~\ref{s:expts} illustrates this in an example. Finally, it is sometimes practical to tag an item as not associated with a certain category, e.g.\ ``this paper is not about regression'' or ``this patient does not have fever'', particularly if this helps to make it distinctive. In summary, in this type of applications, \emph{it is impractical for an item to be fully labeled over all categories, but it is natural for it to be associated or disassociated with a few categories}. This can be coded with item-category similarity values that are positive or negative, respectively, with the magnitude indicating the degree of association, and zero meaning indifference or ignorance. We call this source of information, which is specific for each item irrespectively of other items, the \emph{wisdom of the expert}.

We also consider another practical source of information. Usually it is easy to construct a similarity of a given item to other items, at least its nearest neighbors. For example, with documents or images, this could be based on a bag-of-words representation. We would expect similar items to have similar assignment vectors, and this can be captured with an item-item similarity matrix and its graph Laplacian. We call this source of information, which is about an item in the context of other items, the \emph{wisdom of the crowd}.

As an example of the interaction of these two sources of information, imagine the following example where the items are conference papers and the categories are authors. We know that author A1 writes papers about regression (usually with author A2) or bioinformatics (usually with author A3). We are sent for review a paper that is about regression (it contains many words about regression) and, we are tipped, by A1. Can we guess its coauthors, i.e., predict its assignments to all existing authors? Based on the crowd wisdom alone, many authors could have written the paper (those who write regression papers). Based on the expert wisdom alone, A2 or A3 may have cowritten the paper (since each of them has coathored other papers with A1). Given both wisdoms, we might expect a high assignment for A2 (and A1) and low for everybody else.

In this paper, we propose a simple model that captures this intuition as a quadratic program. We give some properties of the solution, define an out-of-sample mapping, derive a training algorithm, and illustrate the model with document and image datasets.

A shorter version of this work appears in a conference paper \citep{CarreirWang14b}.

\section{The Laplacian assignment (LASS) model}

We consider the following assignment problem. We have $N$ items and $K$ categories, and we want to determine soft assignments $z_{nk}$ of each item $n$ to each category $k$, where $z_{nk} \in [0,1]$ and $\sum^K_{k=1}{z_{nk}} = 1$ for each $n=1,\dots,N$. We are given two similarity matrices, suitably defined, and typically sparse: an item-item similarity matrix \W, which is an $N\times N$ matrix of affinities $w_{nm} \ge 0$ between each pair of items $n$ and $m$; and an item-category similarity matrix \G, which is an $N\times K$ matrix of affinities $g_{nk} \in \bbR$ between each pair of item $n$ and category $k$ (negative affinities, i.e., dissimilarities, are allowed in \G).

We want to assign items to categories optimally as follows:
\begin{subequations}
  \label{e:lass}
  \begin{align}
    \label{e:lass-objfcn}
    \min_{\Z} & \quad \lambda \trace{\Z^T \LL \Z}  - \trace{\G^T \Z} \\
    \label{e:lass-c1}
    \text{s.t.} & \quad \Z \1_K = \1_N \\
    \label{e:lass-c2}
    & \quad \Z \ge \0
  \end{align}
\end{subequations}
where $\lambda > 0$, $\1_K$ is a vector of $K$ ones, and \LL\ is the $N\times N$ graph Laplacian matrix, obtained as $\LL = \D - \W$, where $\D = \diag{\smash{\sum^N_{n=1}{w_{nm}}}}$ is the degree matrix of the weighted graph defined by \W. The problem is a quadratic program (QP) over an $N\times K$ matrix $\Z = (\z_1,\dots,\z_K)$, i.e., $NK$ variables%
\footnote{We will use a boldface vector $\z_n$ or to mean the $n$th row of \Z\ (as a column vector), and a boldface vector $\z_k$ to mean the $k$th column of \Z\ (likewise for $\g_n$ or $\g_k$). The context and the index ($n$ or $k$) will determine which is the case. This mild notation abuse will simplify the explanations.}
altogether, where $\z_k$, the $k$th column of \Z, contains the assignments of each item to category $k$. We will call problems of the type~\eqref{e:lass} \emph{Laplacian assignment problems (LASS)}. Minimizing objective~\eqref{e:lass-objfcn} encourages items to be assigned to categories with which they have high similarity (the linear term in \G), while encouraging similar items to have similar assignments (the Laplacian term in \LL), since
\begin{equation*}
  \trace{\Z^T \LL \Z} = \frac{1}{2} \sum^N_{n,m=1}{w_{nm} \norm{\z_n-\z_m}^2}
\end{equation*}
where $\z_n$ is the $n$th row of \Z, i.e., the assignments for item $n$. Although we could absorb $\lambda$ inside \G, we will find it more convenient to fix the scale of each similarity $g_{nk}$ to, say, the interval $[-1,1]$ (where $\pm 1$ mean maximum (dis)similarity and $0$ ignorance), and then let $\lambda$ control the strength of the Laplacian term.

The objective function~\eqref{e:lass-objfcn} is separable over categories as
\begin{equation*}
  \sum^K_{k=1}{E(\z_k;\LL,\g_k)} \qquad \text{where} \qquad E(\z;\LL,\g) = \lambda \z^T \LL \z - \g^T \z
\end{equation*}
and the constraints~\eqref{e:lass-c1}--\eqref{e:lass-c2} are separable over items as $\z^T_n \1_K = \1_N$, $\z_n \ge \0$, for $n = 1,\dots,N$. Thus, the problem~\eqref{e:lass} is not separable, since all the assignments are coupled with a certain structure.

\subsection{Extreme values of $\lambda$}

We can determine the solution(s) for the following extreme values of $\lambda$:
\begin{itemize}
\item If $\lambda = 0$, then the LASS problem is a linear program (LP) and separates over each item $n = 1,\dots,N$. The solution is $z_{nk} = \delta(k-k_{\text{max}}(n))$ where $k_{\text{max}}(n) = \arg\max\{g_{nk},\ k=1,\dots,K\}$, i.e., each item is assigned to its most similar category. This tells us what the linear term can do by itself. (If the maximum is achieved for more than one category for a given point, then there is an infinite number of ``mixed'' solutions that correspond to giving any assignment value to those categories, and zero for the rest.)
\item If $\lambda = \infty$ or equivalently $\G = \0$, then the LASS problem is a quadratic program with an infinite number of solutions of the form $\z_n = \z_m$ for each $n,m = 1,\dots,N$, i.e., all items have the same assignments. This tells us what the Laplacian term can do by itself.
\item If $\lambda\rightarrow\infty$, i.e., for very large $\lambda$ but still having the linear term, the behavior actually differs from that of $\G = \0$. In the generic case, we expect a unique solution close to $\Z = \1_N \z^T$ where $\z \in \bbR^K$ and $z_k = \delta(k-k_{\text{max}})$ where $k_{\text{max}} = \arg\max_k{(\G^T\1_N)}$, $k = 1,\dots,K$, i.e., all items are assigned to the same category, the one having maximum total similarity $\g^T_k \1_N$ over all items. (Again, if $\G^T\1_N$ has more than one maximum values, there is an infinite number of solutions corresponding to mixed assignments.) Indeed, if $\lambda$ is very large, the Laplacian term dominates and we have that $\z_n = \z_m = \z$ for every pair of items, approximately. Then the LASS problem becomes the following LP
  \begin{align*}
    \max_{\z} & \quad \z^T (\G^T \1_N) \\
    \text{s.t.} & \quad \z^T \1_K = 1 \\
    & \quad \z \ge \0
  \end{align*}
  whose solution allocates all the assignment mass to the category with largest total similarity $\g^T_k \1_N$.
\end{itemize}
With intermediate $\lambda > 0$, more interesting solutions appear (particularly when the similarity matrices are sparse), where the item-category similarities are propagated to all points through the item-item similarities.

\subsection{Existence and unicity of the solution}
\label{s:unicity}

The LASS problem is a convex QP, so general results of convex optimization tell us that all minima are global minima. However, since the Hessian of the objective function is positive semidefinite, there can be multiple minima. The following theorem characterizes the solutions, and its corollary gives a sufficient condition for the minimum to be unique.
\begin{thm}
  \label{th:lass-sol}
  Assume the graph Laplacian \LL\ corresponds to a connected graph and let $\Z^* \in \bbR^{N\times K}$ be a solution (minimizer) of the LASS problem~\eqref{e:lass}. Then, any other solution has the form $\Z^* + \1_N \p^T$ where $\p \in \bbR^K$ satisfies the conditions:
  \begin{equation}
    \label{e:lass-sol}
    \p^T \1_K = 0, \qquad \p^T (\G^T \1_N) = 0, \qquad \Z^* + \1_N \p^T \ge \0.
  \end{equation}
  In particular, for each $k = 1,\dots,K$ for which $\exists n \in \{1,\dots,N\}\mathpunct{:}\ z^*_{nk} = 0$, then $p_k \ge 0$.
\end{thm}
\begin{proof}
  Call $q(\Z) = \lambda \trace{\Z^T \LL \Z}  - \trace{\G^T \Z}$ the objective function of the LASS problem. Since $q$ is continuous and the feasible set of the problem is bounded and closed in $\bbR^{N\times K}$, $q$ achieves a minimum value in the feasible set, hence at least one solution exists, which makes the theorem statement well defined. We call this solution $\Z^*$. Now let us show that $q(\Z^*+\PP) \ge q(\Z^*)$ for any other feasible point $\Z^* + \PP$, with $\PP \in \bbR^{N\times K}$. Simple algebra shows that
  \begin{equation}
    \label{e:lass-sol-eq1}
    q(\Z^* + \PP) = q(\Z^*) + \trace{\PP^T (2\lambda\LL \Z^* - \G)} + \lambda \trace{\PP^T \LL \PP}.
  \end{equation}
  The last term is nonnegative because \LL\ is positive semidefinite. The penultimate term is also nonnegative. To see this, write the KKT conditions for the LASS problem with Lagrange multipliers $\bpi\in\bbR^N$ and $\M\in\bbR^{N\times K} = (\mu_{nk})$ for the equality and inequality constraints, respectively:
  \begin{align*}
    2 \lambda \LL \Z^* - \G - \bpi \1^T_K - \M &= \0 \\
    \Z^* \1_K &= \1_N \\
    \Z^* &\ge \0 \\
    \M &\ge \0 \\
    \mu_{nk} z^*_{nk} &= 0,\ n=1,\dots,N,\ k=1,\dots,K.
  \end{align*}
  Thus, since $\Z^* + \PP$ is feasible, $(\Z^* + \PP) \1_K = \1_N \Rightarrow \PP \1_K = \0$, and $\Z^* + \PP \ge \0 \Rightarrow p_{nk} \ge 0$ where $z^*_{nk} = 0$ (i.e., the active inequalities). Then, from the first KKT equation, we have for the penultimate term:
  \begin{equation*}
    \trace{\PP^T (2\lambda\LL \Z^* - \G)} = \trace{\PP^T \bpi \1^T_K + \PP^T \M} = \trace{\1^T_K \PP^T \bpi + \PP^T \M} = \sum^{N,K}_{n,k=1}{p_{nk} \mu_{nk}} = \sum_{\text{active } n,k}{p_{nk} \mu_{nk}} \ge 0
  \end{equation*}
  because $\mu_{nk} \ge 0$ for all $n,k$, $\mu_{nk} = 0$ for the inactive inequalities, and $p_{nk} \ge 0$ for the active inequalities. Hence, the last two terms in~\eqref{e:lass-sol-eq1} are nonnegative, so $q(\Z^* + \PP) \ge q(\Z^*)$ and $\Z^*$ is a global minimizer.

  Now assume that $q(\Z^* + \PP) = q(\Z^*)$. Then the last two terms in~\eqref{e:lass-sol-eq1} must both be zero. Recall that if the graph Laplacian \LL\ corresponds to a connected graph, it has a single null eigenvalue with an eigenvector consisting of all ones. From $\trace{\PP^T \LL \PP} = 0$ it follows that $\PP = \1_N \p^T$ for some $\p \in \bbR^K$. Since $\Z^* + \PP$ is feasible, $\PP \1_K = \0 \Rightarrow \p^T \1_K = 0$, and $\Z^* + \1_N \p^T \ge \0$. From the penultimate term, $0 = \trace{\PP^T (2\lambda\LL \Z^* - \G)} = \trace{2\lambda \p \1^T_N \LL \Z^* - \p \1^T_N \G} \Rightarrow \p^T (\G^T \1_N) = 0$.

  Finally, from $\Z^* + \1_N \p^T \ge \0$ it follows that $z^*_{nk} + p_{nk} \ge 0$ for each $n = 1,\dots,N$ and $k = 1,\dots,K$, so if $\exists n,k\mathpunct{:}\ z^*_{nk} = 0$ (i.e., if any of the inequalities involving $k$ are active), then $p_k \ge 0$.
\end{proof}
\begin{cor}
  \label{cor:lass-sol}
  Assume the graph Laplacian \LL\ corresponds to a connected graph and let $\Z^* \in \bbR^{N\times K}$ be a solution of the LASS problem~\eqref{e:lass}. If $\max_k{ (\min_n{ (z^*_{nk}) }) } = 0$ then the solution $\Z^*$ is unique.
\end{cor}
\begin{proof}
  We have $\max_k{ (\min_n{ (z^*_{nk}) }) } = 0$, so for each $k$, $\exists n\mathpunct{:}\ z^*_{nk} = 0$. From theorem~\ref{th:lass-sol}, any other solution $\Z^* + \1_N \p^T$ must have $p_k \ge 0$ for $k = 1,\dots,K$, and $\p^T \1_K = 0$. Hence $\p = \0$.
\end{proof}
\begin{rmk}
  If the graph Laplacian \LL\ corresponds to a graph with multiple connected components, then the LASS problem separates into a problem for each component, and the previous theorem holds in each component. Computationally, it is also more efficient to solve each problem separately.
\end{rmk}
\begin{rmk}
  The set~\eqref{e:lass-sol} of solutions to a LASS problem is a convex polytope.
\end{rmk}
\begin{rmk}
  The condition of corollary~\ref{cor:lass-sol} means that each category has at least one item with a zero assignment to it. In practice, we can expect this condition to hold, and therefore the solution to be unique, if the categories are sufficiently distinctive and $\lambda$ is small enough. Equivalently, nonunique solutions arise if some categories are a catch-all for the entire dataset. Theoretically, this should always be possible if $\lambda$ is large enough, particularly if there are many categories. However, in practice we have never observed nonunique solutions, because for large $\lambda$ the algorithm is attracted towards a solution where one category dominates, so the vector \p\ can only have one negative component, which makes~\eqref{e:lass-sol} impossible unless \G\ takes a special value, such as $\G = \0$. Thus, the symmetric situation where all assignments are possible for $\lambda\rightarrow\infty$ does not seem to occur in practice.
\end{rmk}
\begin{rmk}
  Practically, one can always make the solution unique by replacing \LL\ with $\LL + \epsilon \I_N$ where $\epsilon > 0$ is a small value, since this makes the objective strongly convex. (This has the equivalent meaning of adding a penalty $\epsilon \smash{\norm{\Z}^2_F}$ to it, which has the effect of biasing the assignment vector $\z_n$ of each item towards the simplex barycenter, i.e., uniform assignments.) However, as noted before, nonunique solutions appear to be rare with practical data, so this does not seem necessary.
\end{rmk}

\subsection{Particular cases}

\begin{thm}
  \label{th:partic}
  Assume the graph Laplacian \LL\ corresponds to a connected graph and let $\Z \in \bbR^{N\times K}$ be a solution of the LASS problem~\eqref{e:lass}. Then:
  \begin{enumerate}
  \item If $\g_n = \0$ then $\z_n = \frac{\Z^T \w_n}{\1^T_N \w_n} = \sum^N_{m=1}{\frac{w_{nm}}{\sum^N_{m'=1}{w_{nm'}}} \z_m}$, $\bmu_n = \0$ and $\pi_n = 0$.
  \item If $\g_k \le \0$ then $\z_k = \0$, $\bmu_n = -\g_k$ and $\pi_n = 0$.
  \end{enumerate}
\end{thm}
\begin{proof}
  Both statements follow from substituting the values in the KKT conditions~\eqref{e:lass-kkt}. Conditions~\eqref{e:lass-kkt2}--\eqref{e:lass-kkt5} are trivially satisfied, so we prove condition~\eqref{e:lass-kkt1} only. For statement 1, we can write row $n$ of \LL\ as $\bl_n = -\w_n + (\1^T_N \w_n) \e_n$, where $\w_n$ is row $n$ of \W\ and $\e_n$ is a vector with entries $e_{nm} = \delta(n-m)$, and we write all row vectors as column vectors. Then we can write row $n$ (in column form) of condition~\eqref{e:lass-kkt1} as:
  \begin{equation*}
    \0 = 2 \lambda \Z^T \bl_n - \g_n - \pi_n \1_K - \bmu_n = 2 \lambda ( -\Z^T \w_n + (\1^T_N \w_n) \z_n ) \Rightarrow \z_n = \frac{\Z^T \w_n}{\1^T_N \w_n}.
  \end{equation*}
  For statement 2, we can write column $k$ of condition~\eqref{e:lass-kkt1} as $\0 = 2 \lambda \LL \z_k - \g_k - \bpi - \bmu_k \Rightarrow \bmu = -\g_k$.
\end{proof}
\begin{rmk}
  The meaning of th.~\ref{th:partic} is as follows. (1) An item for which no similarity to any category is given (i.e., no expert information) receives as assignment the average of its neighbors. This corresponds to the SSL prediction. (2) A category for which no item has a positive similarity receives no assignments.
\end{rmk}

\subsection{Lagrange multipliers for a solution}

Given a feasible point $\Z_{N\times K}$ in parameter space, we may want to test whether it is a solution of the LASS problem. For a QP, the KKT conditions are necessary and sufficient for a solution \citep{NocedalWright06a}. For our problem, and written in matrix form, the KKT conditions are:
\begin{subequations}
  \label{e:lass-kkt}
  \begin{align}
    \label{e:lass-kkt1}
    2 \lambda \LL \Z - \G - \bpi \1^T_K - \M &= \0 \\
    \label{e:lass-kkt2}
    \Z \1_K &= \1_N \\
    \label{e:lass-kkt3}
    \Z &\ge \0 \\
    \label{e:lass-kkt4}
    \M &\ge \0 \\
    \label{e:lass-kkt5}
    \M \circ \Z &= \0
  \end{align}
\end{subequations}
where $\circ$ means elementwise product, and \bpi\ and \M\ are the Lagrange multipliers associated with the point \Z\ for the equality and inequality constraints, respectively. We need to compute \bpi\ and \M. Given \Z, the KKT system~\eqref{e:lass-kkt} has $2NK$ linear equations for $N + NK$ unknowns, and its solution is unique if \Z\ is feasible, as we will see. To obtain it, we multiply~\eqref{e:lass-kkt1} times $\1_K$ on the right and obtain \bpi\ as a function of \M:
\begin{equation}
  \label{e:lass-eqL}
  \bpi = \frac{1}{K} (2 \lambda \LL \Z - \G - \M) \1_K = - \frac{1}{K} (\G+\M) \1_K.
\end{equation}
Substituting it in~\eqref{e:lass-kkt1} gives, together with~\eqref{e:lass-kkt5}:
\begin{gather*}
  \M = \Q + \frac{1}{K} \M \1_K \1^T_K \\
  \M \circ \Z = \0
\end{gather*}
where $\Q = 2 \lambda \LL \Z - \G + \frac{1}{K} \G \1_K \1^T_K$. This is a linear system of $2NK$ equations for the $NK$ unknowns in \M. It separates over each row of \M, $n = 1,\dots,N$, in a system of the form
\begin{equation*}
  \begin{array}{r}
    \bmu - \frac{1}{K} (\1^T_K \bmu) \1_K = \q \\
    \diag{\bmu} \diag{\z} = \0    
  \end{array}
  \qquad \Longleftrightarrow \qquad
  \begin{pmatrix}
    \I_K - \frac{1}{K} \1^T_K \1_K \\ \diag{\z}
  \end{pmatrix}
  \bmu =
  \begin{pmatrix}
    \q \\ \0
  \end{pmatrix}
\end{equation*}
where \bmu, \z\ and \q\ correspond to the $n$th row of \M, \Z\ and \Q, respectively, written as $K$-dimensional column vectors (we omit the index $n$). This is a linear system of $2K$ equations for $K$ unknowns, which we solve by multiplying on the left times the transpose of the coefficient matrix:
\begin{equation*}
  \left( \bPsi - \frac{1}{K} \1_K \1^T_K \right) \bmu = \q - \frac{1}{K} (\1^T_K \q) \1_K \qquad \bPsi = \diag{\z}^2 + \I_K.
\end{equation*}
Finally, using the matrix inversion lemma
\begin{equation*}
  (\A + \alpha \uu \uu^T)^{-1} = \A^{-1} - \frac{\alpha \A^{-1} \uu \uu^T \A^{-1}}{1 + \alpha \uu^T \A^{-1} \uu}
\end{equation*}
we obtain the solution for \bmu:
\begin{equation}
  \label{e:lass-ineqL}
  \bmu = \bPsi^{-1} \q - \left( \frac{\1^T_K \q - \1^T_K \bPsi^{-1} \q}{K - \1^T_K \bPsi^{-1} \1_K} \right) \bPsi^{-1} \1_K
\end{equation}
whose transpose gives row $n$ of \M. Note the formula is well defined because $\psi_{kk} = z^2_k + 1 \ge 0$ and $\1^T_K \bPsi^{-1} \1_K = \sum^K_{k=1}{\frac{1}{1 + z^2_k}} < K$ since $z_k \ge 0$ for $k=1,\dots,K$ and $z_k > 0$ for at least one $k \in \{1,\dots,K\}$, since \Z\ is feasible.

For the case $\lambda = 0$, one can verify that the above formulas simplify as follows (again, we give \z, \bmu\ and $\pi$ for item $n$ but omitting the index $n$):
\begin{align*}
  \z &= (z_k),\ z_k = \delta(k-k_{\text{max}}) \\
  \bmu &= g_{\text{max}} \1_K - \g \\
  \pi &= -g_{\text{max}}
\end{align*}
where \g\ represents the $n$th row of \G, $k_{\text{max}} = \arg\max\{g_k,\ k=1,\dots,K\}$ and $g_{\text{max}} = g_{k_{\text{max}}}$.

\section{A simple, efficient algorithm to solve the QP}

It is possible to solve problem~\eqref{e:lass} in different ways, but one must be careful in designing an effective algorithm because the number of variables and the number of constraints grows proportionally to the number of data points, and can then be very large. We describe here one algorithm that is very simple, has guaranteed convergence without line searches, and takes advantage of the structure of the problem and the sparsity of \LL. It is based on the alternating direction method of multipliers (ADMM), combined with a direct linear solver using the Schur's complement and caching the Cholesky factorization of \LL.

\subsection{QP solution using ADMM}

We briefly review how to solve a QP using the alternating direction method of multipliers (ADMM), following \citep{Boyd_11a}. Consider the QP
\begin{align}
  \label{e:qp}
  \min_{\x} & \quad \frac{1}{2} \x^T \PP \x + \q^T\x \\
  \text{s.t.} & \quad \A\x = \b,\ \x \ge \0
\end{align}
over $\x \in \bbR^D$, where \PP\ is positive (semi)definite. In ADMM, we introduce new variables $\z \in \bbR^D$ so that we replace the inequalities with an indicator function $g(\z)$, which is zero in the nonnegative orthant $\z \ge \0$ and $\infty$ otherwise. Then we write the problem as
\begin{align}
  \label{e:qp-admm}
  \min_{\x} & \quad f(\x) + g(\z) \\
  \text{s.t.} & \quad \x = \z
\end{align}
where
\begin{equation*}
  f(\x) = \frac{1}{2} \x^T \PP \x + \q^T\x,\qquad \dom{f} = \{\x\in\bbR^D\mathpunct{:}\ \A\x = \b\}
\end{equation*}
is the original objective with its domain restricted to the equality constraint. The augmented Lagrangian is
\begin{equation}
  \label{e:auglag}
  \calL(\x,\z,\y;\rho) = f(\x) + g(\z) + \y^T (\x-\z) + \frac{\rho}{2} \norm{\x-\z}^2
\end{equation}
and the ADMM iteration has the form:
\begin{align*}
  \x &\leftarrow \arg\min_{\x}{\calL(\x,\z,\y;\rho)} \\
  \z &\leftarrow \arg\min_{\z}{\calL(\x,\z,\y;\rho)} \\
  \y &\leftarrow \y + \rho(\x - \z)
\end{align*}
where \y\ is the dual variable (the Lagrange multiplier estimates for the constraint $\x = \z$), and the updates are applied in order and modify the variables immediately. Here, we use the scaled form of the ADMM iteration, which is simpler. It is obtained by combining the linear and quadratic terms in $\x-\z$ and using a scaled dual variable $\uu = \y/\rho$:
\begin{align*}
  \x &\leftarrow \arg\min_{\x}{\left( f(\x) + \frac{\rho}{2} \norm{\x - \z + \uu}^2 \right)} \\
  \z &\leftarrow \arg\min_{\z}{\left( g(\z) + \frac{\rho}{2} \norm{\x - \z + \uu}^2 \right)} \\
  \uu &\leftarrow \uu + \x - \z.
\end{align*}
Since $g(\x)$ is the indicator function for the nonnegative orthant, the solution of the \z-update is simply to threshold each entry in $\x+\uu$ by taking is nonnegative part. Finally, the ADMM iteration is:
\begin{align}
  \label{e:admm}
  \x &\leftarrow \arg\min_{\x}{\left( f(\x) + \frac{\rho}{2} \norm{\x - \z + \uu}^2 \right)} \\
  \z &\leftarrow (\x + \uu)_+ \\
  \uu &\leftarrow \uu + \x - \z
\end{align}
where the updates are applied in order and modify the variables immediately, and $(t)_+ = \max(t,0)$ applies elementwise, and $\norm{\cdot}$ is the Euclidean norm. The penalty parameter $\rho > 0$ is set by the user, and $\z = (z_1,\dots,z_D)^T$ are the Lagrange multiplier estimates for the inequalities. The \x-update is an equality-constrained QP with KKT conditions
\begin{equation}
  \label{e:qp-kkt}
  \begin{pmatrix}
    \PP + \rho \I & \A^T \\ \A & \0
  \end{pmatrix}
  \begin{pmatrix}
    \x \\ \bnu
  \end{pmatrix}
  =
  \begin{pmatrix}
    -\q + \rho(\z-\uu) \\ \b
  \end{pmatrix}.
\end{equation}
Solving this linear system gives the optimal \x\ and \bnu\ (the Lagrange multipliers for the equality constraint). The ADMM iteration consists of very simple updates to the relevant variables, but its success crucially relies in being able to solve the \x-update efficiently. Given the structure of our problem, it is convenient to use a direct solution using Schur's complement, that is:
\begin{subequations}
  \label{e:qp-schur}
  \begin{align}
    \label{e:qp-schur1}
    (\A (\PP + \rho \I)^{-1} \A^T) \bnu &= \A (\PP + \rho \I)^{-1} (-\q + \rho(\z-\uu) ) + \b \\
    \label{e:qp-schur2}
    (\PP + \rho \I)^{-1} \x &= -\A^T \bnu -\q + \rho(\z-\uu).
  \end{align}
\end{subequations}
Eq.~\eqref{e:qp-schur1} results from left-multiplying the first equation in~\eqref{e:qp-kkt} by $\A (\PP + \rho \I)^{-1}$ and using the second equation in~\eqref{e:qp-kkt} to eliminate \x. Eq.~\eqref{e:qp-schur2} results from substituting \bnu\ back in the first equation in~\eqref{e:qp-kkt} and solving it for \x.

Thus, the ADMM iteration consists of solving a linear system on the primal variables \x, applying a thresholding to get \z, and an addition to get \uu. Convergence of the ADMM iteration~\eqref{e:admm} to the global minimum of problem~\eqref{e:qp} in value and to a feasible point is guaranteed for any $\rho > 0$.

\subsection{Application to our QP}

We now write the ADMM updates for our QP~\eqref{e:lass}, where we identify:
\begin{align*}
  \PP &= 2 \lambda \diag{\LL,\dots,\LL} \text{ of } NK \times NK \\
  \q &= -\vect{\G} \text{ of } NK \times 1 \\
  \A &= (\I_N,\dots,\I_N) \text{ of } N \times NK \\
  \b &= \1_N \text{ of } N \times 1
\end{align*}
where $\vect{\cdot}$ concatenates the columns of its argument into a single column vector. Given the structure of these matrices, the solution of the KKT system~\eqref{e:qp-kkt} by using Schur's complement~\eqref{e:qp-schur} simplifies considerably. The basic reasons are that (1) the matrix \PP\ is block-diagonal with $K$ identical copies of the graph Laplacian \LL, which is itself usually sparse; and (2) the especially simple form of the equality constraint matrix \A. Thus, even though the \x-update involves solving a large linear system of $NK$ equations, it is equivalent to solving $K$ systems of $N$ equations where the coefficient matrix is the same for each system and besides is constant and sparse, equal to $2\lambda\LL + \rho\I$. In turn, these linear systems may be solved efficiently in one of the two following ways: (1) preferably, by caching the Cholesky factorization of this matrix (using a good permutation to reduce fill-in), if it does not add so much fill that it can be stored; or (2) by using an iterative linear solver such as conjugate gradients, initialized with a warm start, preconditioned, and exiting it before convergence, so as to carry out faster, inexact \x-updates.

The final algorithm is as follows, with its variables written in matrix form. The input are the affinity matrices $\G_{N\times K}$ and $\W_{N\times N}$, from which we construct the graph Laplacian $\LL_{N\times N}$. We then choose $\rho > 0$ and set
\begin{equation*}
  \h = -\frac{1}{K} \G\1_K + \frac{\rho}{K} \1_N \qquad \RR\RR^T = 2\lambda\LL + \rho\I.
\end{equation*}
The Cholesky factor \RR\ is used to solve linear system~\eqref{e:lass-admm2}. We then iterate, in order, the following updates until convergence:
\begin{subequations}
  \label{e:lass-admm}
  \begin{align}
    \label{e:lass-admm1}
    \bnu &\leftarrow \frac{\rho}{K} (\Y - \U) \1_K - \h \\
    \label{e:lass-admm2}
    \Z &\leftarrow (2\lambda\LL + \rho\I)^{-1} ( \rho (\Y - \U) + \G - \bnu\1^T_K) \\
    \label{e:lass-admm3}
    \Y &\leftarrow (\Z + \U)_+ \\
    \label{e:lass-admm4}
    \U &\leftarrow \U + \Z - \Y
  \end{align}
\end{subequations}
where $\Z_{N\times K}$ are the primal variables, $\Y_{N\times K}$ the auxiliary variables, $\U_{N\times K}$ the Lagrange multiplier estimates for \Y, and $\bnu_{N\times 1}$ the Lagrange multipliers for the equality constraints. The solution for the linear system in the \Z-update may be obtained by using two triangular backsolves if using the Cholesky factor of $2\lambda\LL + \rho\I$, or using an iterative method such as conjugate gradients if the Cholesky factor is not available.

The iteration~\eqref{e:lass-admm} is very simple to implement. It requires no line searches and has only one user parameter, the penalty parameter. The algorithm converges for any positive value of the penalty parameter, but this value does affect the convergence rate.

\subsection{Remarks}

\begin{thm}
  \label{th:admm-kkt24}
  At each iterate in the algorithm updates~\eqref{e:lass-admm}, $\Z\1_K = \1_N$, $\U \le \0$, $\Y \ge \0$ and $\Y \circ \U = \0$.
\end{thm}
\begin{proof}
  For \Z, substituting eq.~\eqref{e:lass-admm1} into~\eqref{e:lass-admm2}:
  \begin{equation*}
    \Z\1_K = (2\lambda\LL + \rho\I)^{-1} ( \rho (\Y - \U) + \G - \bnu\1^T_K) \1_K = (2\lambda\LL + \rho\I)^{-1} ( \rho (\Y - \U) \1_K + \G \1_K - K \bnu =(2\lambda\LL + \rho\I)^{-1} \rho \1_N = \1_N
  \end{equation*}
  where the last step results from $(2\lambda\LL + \rho\I) \1_N = 2\lambda\LL \1_N + \rho \1_N =\rho \1_N$. For \U, from eqs.~\eqref{e:lass-admm3}--\eqref{e:lass-admm4} we have that $\U \leftarrow \U+\Z - (\U+\Z)_+ = (\U+\Z)_- \le \0$. For \Y, $\Y \ge \0$ follows from $\Y = (\Z + \U)_+$ in~\eqref{e:lass-admm3}. Finally, $\Y \circ \U = \0$ follows from $\Y = (\Z + \U)_+$ and $\U = (\U+\Z)_-$.
\end{proof}
\begin{thm}
  Upon convergence of algorithm~\eqref{e:lass-admm}, \Z\ is a solution with Lagrange multipliers $\bpi = -\bnu$ and $\M = -\rho\U$.
\end{thm}
\begin{proof}
  Let us compare the KKT conditions~\eqref{e:lass-kkt} with the algorithm updates~\eqref{e:lass-admm} upon convergence, i.e., at a fixed point of the update equations. From th.~\ref{th:admm-kkt24} we know that $\Z\1_K = \1_N$ and $\U \le \0 \Rightarrow \M = -\rho\U \ge \0$, which are KKT conditions~\eqref{e:lass-kkt2} and~\eqref{e:lass-kkt4}. From eq.~\eqref{e:lass-admm4} we must have $\Z = \Y$, so from eq.~\eqref{e:lass-admm3} we have $\Z = (\Z+\U)_+\ge \0$, which is KKT condition~\eqref{e:lass-kkt3}. From eqs.~\eqref{e:lass-admm3}--\eqref{e:lass-admm4} we have $\Z = (\Z+\U)_+$ and $\U = (\Z+\U)_-$, therefore $\U \circ \Z = \M \circ \Z = \0$, which is KKT condition~\eqref{e:lass-kkt5}. Finally, from eq.~\eqref{e:lass-admm2} we have:
  \begin{equation*}
    (2\lambda\LL + \rho\I) \Z = \rho (\Y - \U) + \G - \bnu\1^T_K \Rightarrow 2\lambda\LL\Z - \G + \rho\U + \bnu\1^T_K = 2\lambda\LL\Z - \G - \bpi\1^T_K - \M = \0
  \end{equation*}
  which matches KKT condition~\eqref{e:lass-kkt1}. The change of sign in the multipliers between the algorithm and the KKT conditions is due to the sign choice in the Lagrangian (adding in eq.~\eqref{e:auglag}, subtracting in~\eqref{e:lass-kkt}).
\end{proof}
\begin{rmk}
  In practice, the algorithm is stopped before convergence, and \Z, $\bpi = -\bnu$ and $\M = -\rho\U$ are estimates for a solution and its Lagrange multipliers, respectively. The estimate \Z\ may not be feasible, in particular the values $z_{nk}$ need not be in $[0,1]$, since this is only guaranteed upon convergence. If needed, a feasible point may be obtained by projecting each row of \Z\ onto the simplex (see section~\ref{s:oos}).
\end{rmk}
\begin{rmk}
  If $\G = \0$ (or $\lambda\rightarrow\infty$) one solution is given by $\Z = \frac{1}{K} \1_N\1^T_K$, for which the Lagrange multipliers are $\M=\0$ and $\bpi=\0$, thus the inequality constraints are inactive and the equality constraints are weakly active. Indeed, that \Z\ value is also a solution of the unconstrained problem.
\end{rmk}

\subsection{Computational complexity}
\label{s:cost}

Each step in~\eqref{e:lass-admm} is $\calO(NK)$ except for the linear system solution in~\eqref{e:lass-admm2}. If \LL\ is sparse, using the Cholesky factor makes this step $\calO(NK)$ as well, and adds a one-time setup cost of computing the Cholesky factor (which is also linear in $N$ with sufficiently sparse matrices). Thus, each iteration of the algorithm is cheap. In practice, for good values of $\rho$, the algorithm quickly approaches the solution in the first iterations and then converges slowly, as is known with ADMM algorithms in general. However, since each iteration is so cheap, we can run a large number of them if high accuracy is needed. As a sample runtime, for a problem with $N=10\,000$ items and $K=10$ categories (i.e., \Z\ has $10^5$ parameters) and using a $100$-nearest-neighbor graph, the Cholesky factorization takes $0.5$~s and each iteration takes $0.15$~s in a PC.

For large-scale problems, the slow convergence becomes more problematic, and it is possible that the Cholesky factor may create too much fill even with a good preordering. One can use instead an iterative linear solver, such as preconditioned conjugate gradients. Scaling up the training is a topic for future research.

\subsection{Initialization}

If the LASS problem is itself a subproblem in a larger problem (as in the Laplacian $K$-modes clustering algorithm; \citealp{WangCarreir14c}), one should warm-start the iteration of eq.~\eqref{e:lass-admm} from the values of \Y\ and \U\ in the previous outer-loop iteration. Otherwise, we can simply initialize $\Y = \U = \0$, which (substituting in eqs.~\eqref{e:lass-admm1}--\eqref{e:lass-admm2}) gives $\Z_0 = (2 \lambda \LL + \rho \I)^{-1} (\G - \frac{1}{K} \G\1_K\1^T_K) + \frac{1}{K} \1_N\1^T_K$ (where $\G - \frac{1}{K} \G\1_K\1^T_K$ is the matrix \G\ with centered rows, and $\frac{1}{K} \1_N\1^T_K$ is the simplex barycenter). This initialization is closely related to the projection on the simplex of the unconstrained optimum of the LASS problem, as we show next. Consider first the unconstrained minimization
\begin{equation*}
  \min_{\Z}{ f(\Z) = \lambda \trace{\Z^T \LL \Z}  - \trace{\G^T \Z} }.
\end{equation*}
This problem is in fact unbounded unless $\G^T\1 = \0$, because taking $\Z = \1\vv^T$ for any $\vv\in\bbR^K$, $\vv \neq \0$, since $\LL\1 = \0$, we have $f(\1\vv^T) = (\G^T\1)^T \vv$, which can be made arbitrarily negative. We could still try to define a \Z\ from the gradient $\nabla f(\Z)^T = 2\lambda \LL \Z - \G$, but this involves a linear system on \LL, whose computational cost defeats the purpose of the initialization. Instead, we can consider the unconstrained minimization
\begin{equation*}
  \min_{\Z} \frac{\rho}{2} \trace{\Z^T \Z} + \lambda \trace{\Z^T \LL \Z}  - \trace{\G^T \Z}
\end{equation*}
for $\rho > 0$, which is strongly convex and has a unique minimum $\Z^* = (2 \lambda \LL + \rho \I)^{-1} \G$, which we can compute cheaply if we reuse the Cholesky factor for $2 \lambda \LL + \rho \I$. Now, we can write the initialization $\Z_0$ (for $\Y = \U = \0$) in terms of $\Z^*$ as $\Z_0 = \Z^* - (\frac{1}{K} \Z^* \1_K) \1^T_K + \frac{1}{K} \1_N\1^T_K$, which means that each row vector of $\Z^*$ is translated along the direction $\1_K$. Since this direction is orthogonal to the simplex, both $\Z^*$ and $\Z_0$ have the same projection on it.

Finally, note that if $\rho$ is large, then $\Z_0 \approx \frac{1}{K} \1_N\1^T_K$ and $\Z^* \approx \0$, both of which project onto the simplex barycenter, independently of the problem data.

\subsection{Stopping criterion}

We stop when $\norm{\smash{\Z^{(\tau+\Delta)}-\Z^{(\tau)}}}_1 <$ \texttt{tol}, i.e., when the change in absolute terms in \Z\ in the last $\Delta$ iterations falls below a set tolerance \texttt{tol} (e.g.\ $10^{-5}$). Using an absolute criterion here is equivalent to using a relative one, since $\norm{\z_n}_1 = 1$, $n=1,\dots,N$. Since our iterations are so cheap, evaluating $\norm{\Z^{(\tau+\Delta)}-\Z^{(\tau)}}_1$ takes a runtime comparable to that of the updates in~\eqref{e:lass-admm} (except for the \Z-update, possibly), so testing the stopping criterion every $\Delta = 10$--$100$ iterations saves around $10$\% runtime.

Another possible stopping criterion is to test whether the KKT conditions~\eqref{e:lass-kkt} are satisfied up to a given tolerance, using the Lagrange multipliers' estimates provided by the algorithm at each iterate. Each iterate always satisfies~\eqref{e:lass-kkt2} and~\eqref{e:lass-kkt4}, so we only need to check~\eqref{e:lass-kkt1}, \eqref{e:lass-kkt3} and~\eqref{e:lass-kkt5} (if the iterate is interior to the inequalities it will also satisfy~\eqref{e:lass-kkt3} and~\eqref{e:lass-kkt5}). Still, it is faster to check for changes in \Z.

Since the iterates \Z\ in the algorithm need not be feasible, they may be slightly infeasible once the stopping criterion is satisfied. If desired, a feasible \Z\ can be obtained by projecting each assignment vector $\z_n$ onto the simplex (see section~\ref{s:oos}).

\subsection{Optimal penalty parameter $\rho$}

The speed at which ADMM converges depends on the quadratic penalty parameter $\rho$ \citep{Boyd_11a}. We illustrate this with the ``2 moons'' dataset in fig.~\ref{f:rho} ($N = 4\,000$ points, $K=2$ categories, $5$-nearest-neighbor graph, $\lambda=1$), where we set positive similarity values for one point in each cluster, resulting in each cluster being assigned to a different category, as expected. The problem has $8\,000$ parameters and we ran $10^4$ iterations, which took 11 s. Little work exists on how to select $\rho$ so as to achieve fastest convergence. Recently, for QPs, \citet{Ghadim_13a} suggest to use $\rho^*=2\lambda\sqrt{\sigma_{\min} \sigma_{\max}}$ where $\sigma_{\min}$ and $ \sigma_{\max}$ are the smallest (nonzero) and largest eigenvalue of the Laplacian. In fig.~\ref{f:rho}, $\rho^*\approx 0.2$, and we show the relative error $\smash{\norm{\Z-\smash{\Z_{\text{opt}}}}_{\text{F}}/\norm{\smash{\Z_{\text{opt}}}}_{\text{F}}}$ vs number of iterations for different $\rho/\rho^*$ (initial $\Z=\0$, with relative error $1$). Asymptotically, the convergence is linear; a model $\epsilon_k = r^k \epsilon_0$ gives $r \approx 0.9977$ for $\rho^*$. While, in the long term, values close to $\rho^*$ work best, in the short term, smaller values are able to achieve an acceptably low relative error ($\approx 10^{-3}$) in just a few iterations, so an adaptive $\rho$ would be best overall.

\begin{figure}[t]
  \centering
  \begin{tabular}{@{}c@{\hspace{0.08\linewidth}}c@{}}
    \begin{tabular}[c]{@{}c@{}}
      \includegraphics[width=0.31\linewidth]{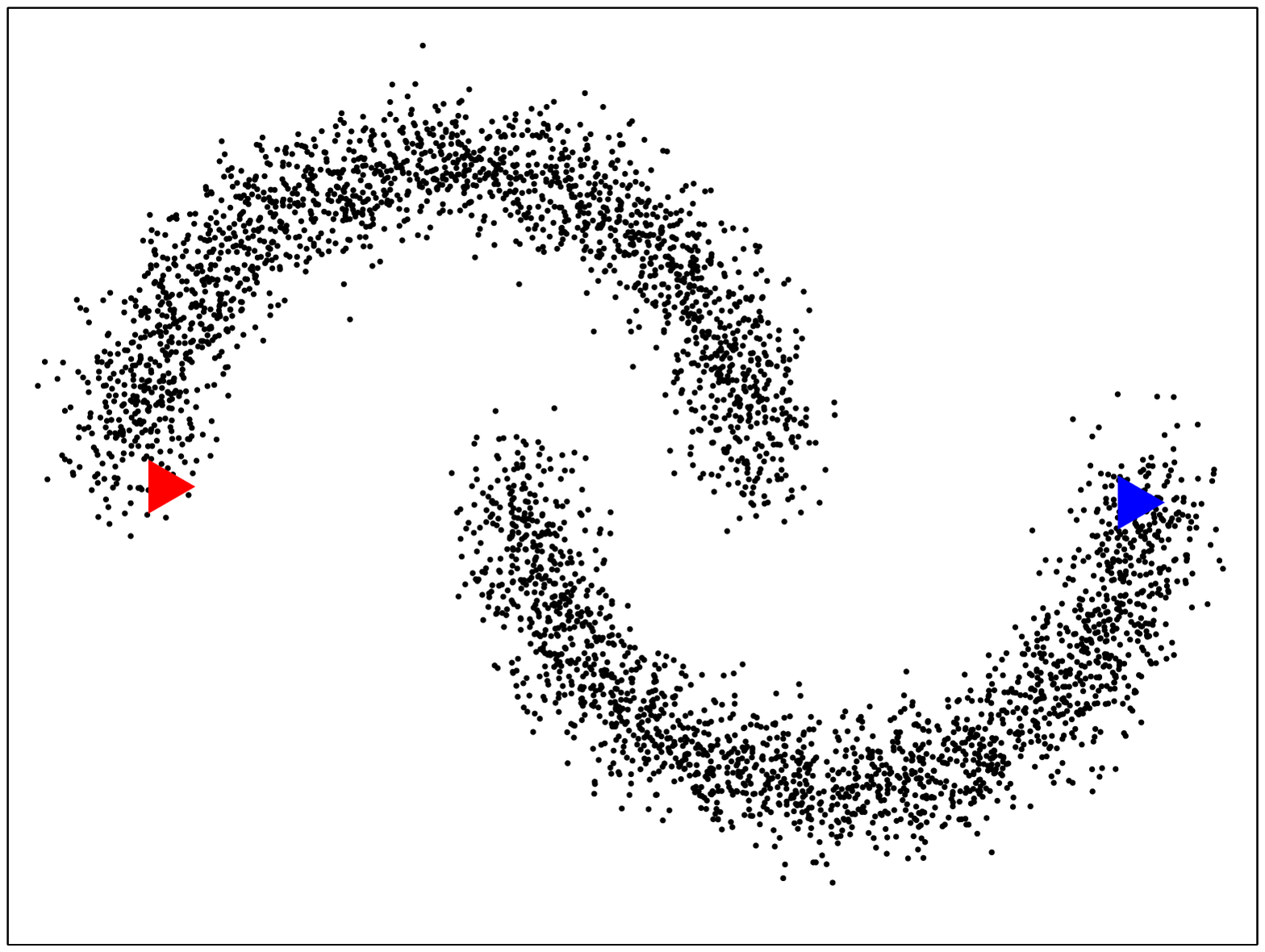} \\
      \includegraphics[width=0.31\linewidth]{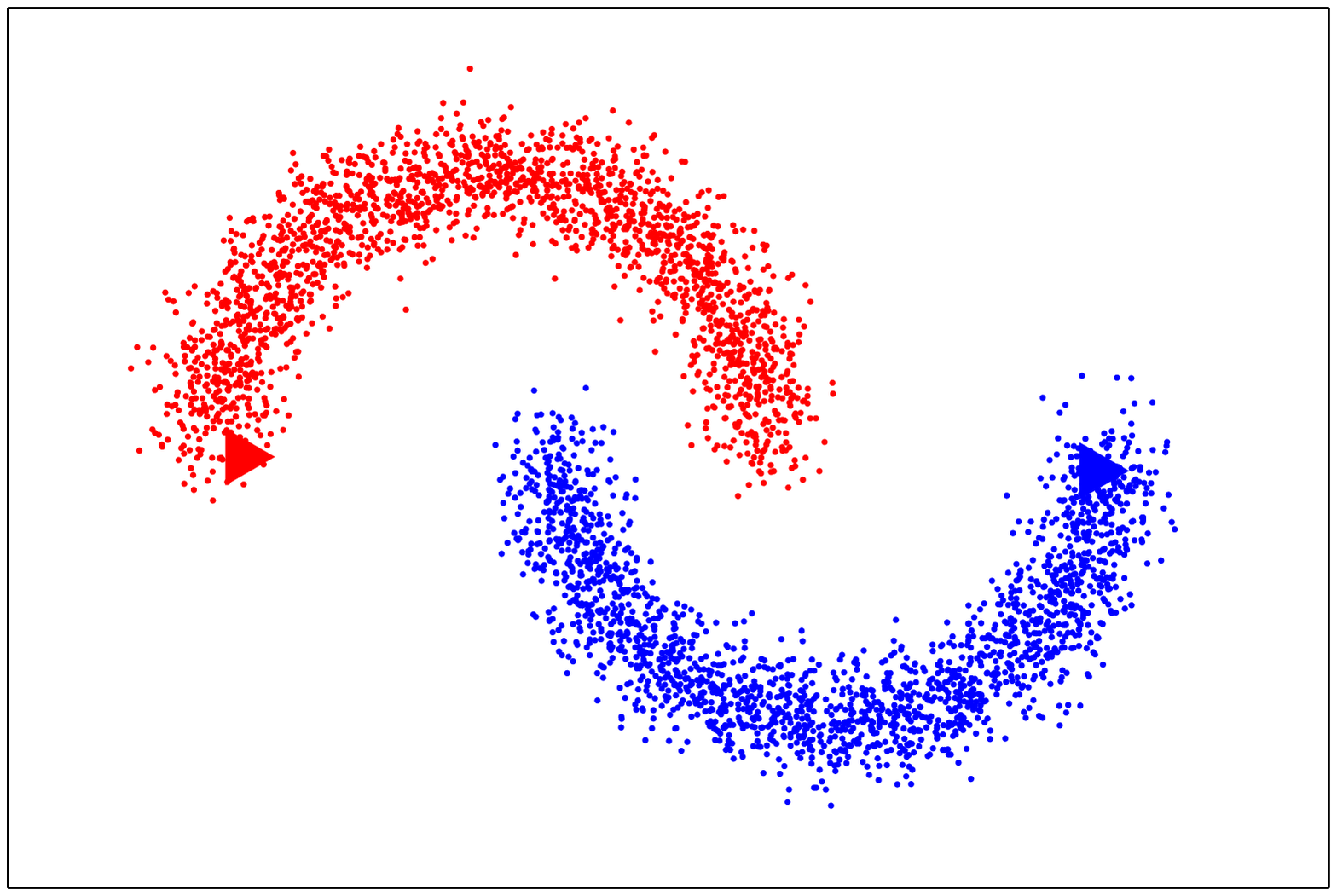}
    \end{tabular} &
    \begin{tabular}[c]{@{}c@{}}
      \psfrag{iter}[t][]{iterations $\times 10^3$}
      \psfrag{error}[][t]{relative error}
      \psfrag{rho=rrr1e-3}[l][l]{$10^{-3}$}
      \psfrag{rho=rrr1e-2}[l][l]{$10^{-2}$}
      \psfrag{rho=rrr1e-1}[l][l]{$10^{-1}$}
      \psfrag{rho=rrr1e0}[l][l]{$10^{0}$}
      \psfrag{rho=rrr1e1}[l][l]{$10^{1}$}
      \psfrag{rho=rrr1e2}[l][l]{$10^{2}$}
      \psfrag{rho=rrr1e3}[l][l]{$10^{3}$}
      \includegraphics[width=0.57\linewidth]{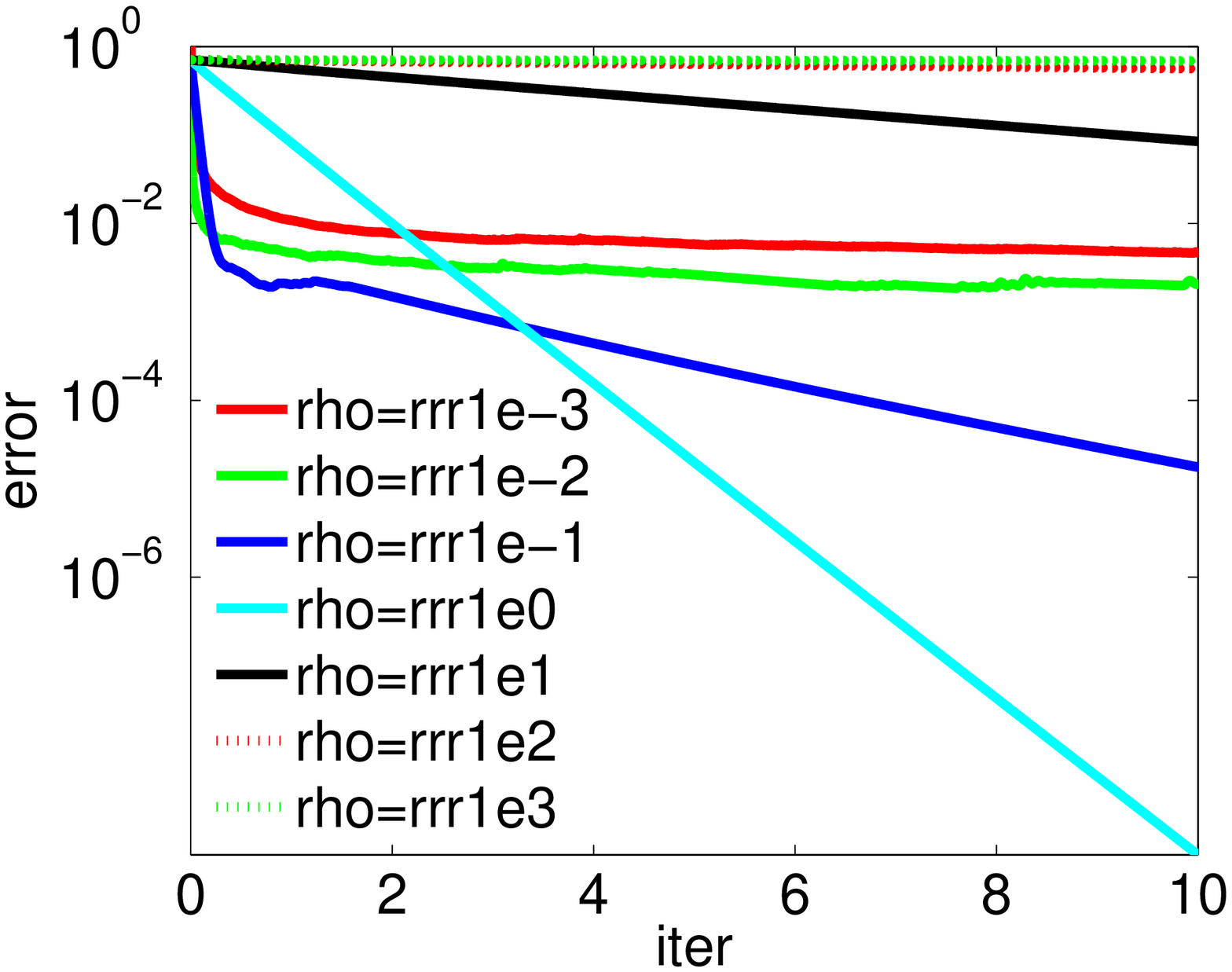}
    \end{tabular}
  \end{tabular}
  \caption{Convergence speed of ADMM for different $\rho$. \emph{Left}: a 2-moons dataset for which only two points are provided with \g\ affinities (above) and predicted assignments (below). \emph{Right}: error (in log scale) vs iterations for different $\rho/\rho^*$.}
  \label{f:rho}
\end{figure}

\subsection{Matlab code}
\label{s:matlab}

The following Matlab code implements the algorithm, assuming a direct solution of the \Z-update linear system.

\begin{verbatim}
function [Z,Y,U,nu] = lass(L,l,G,r,Y,U,maxit,tol)

[N,K] = size(G); LI = 2*l*L+r*speye(N,N); h = (-sum(G,2)+r)/K; Zold = zeros(N,K);
for i=1:maxit
  nu = (r/K)*sum(Y-U,2) - h;
  Z = LI \ bsxfun(@minus,r*(Y-U)+G,nu);
  Y = max(Z+U,0);
  U = U + Z - Y;
  if max(abs(Z(:)-Zold(:))) < tol break; end; Zold = Z;
end
\end{verbatim}

\section{Out-of-sample mapping}
\label{s:oos}

Having trained the system, that is, having found the optimal assignments \Z\ for the training set items, we are given a new, test item \x\ (for example, a new point $\x \in \bbR^D$), along with its item-item and item-category similarities $\w = (w_n)$, $n=1,\dots,N$ and $\g = (g_k)$, $k=1,\dots,K$, respectively, and we wish to find its assignment to each category. We follow the reasoning of \citet{CarreirLu07a} to derive an out-of-sample mapping. While one could train again the whole system augmented with \x, this would be very time-consuming, and the assignments of all points would change (although very slightly). A more practical and still natural way to define an out-of-sample mapping $\z(\x)$ is to solve a problem of the form~\eqref{e:lass} with a dataset consisting of the original training set augmented with \x, but keeping \Z\ fixed to the values obtained during training. Hence, the only free parameter is the assignment vector \z\ for the new point \x. After dropping constant terms, the optimization problem~\eqref{e:lass} reduces to the following quadratic program over $K$ variables:
\begin{subequations}
  \label{e:lass-oos}
  \begin{align}
    \label{e:lass-oos1}
    \min_{\z} & \quad \norm{\z - (\bar{\z} + \gamma\g)}^2 \\
    \label{e:lass-oos2}
    \text{s.t.} & \quad \z^T \1_K = 1 \\
    \label{e:lass-oos3}
    & \quad \z \ge \0
  \end{align}
\end{subequations}
where $\gamma = 1/2\lambda(\1^T_N\w) = 1/2\lambda \sum^N_{n=1}{w_n}$ and
\begin{equation*}
  \bar{\z} = \frac{\Z^T \w}{\1^T_N\w} = \sum^N_{n=1}{\frac{w_n}{\sum^N_{n'=1}{w_{n'}}} \z_n}
\end{equation*}
is a weighted average of the training points' assignments, and so $\bar{\z} + \gamma\g$ is itself an average between this and the item-category affinities. Thus, the solution is the Euclidean projection $\Pi(\bar{\z} + \gamma\g)$ of the $K$-dimensional vector $\bar{\z}+\gamma\g$ onto the probability simplex. This can be efficiently computed, in a finite number of steps, with a simple $\calO(K \log{K})$ algorithm \citep{Duchi_08a,WangCarreir14a}. Computationally, assuming \w\ is sparse, the most expensive step is finding the neighbors to construct \w. With large $N$, one should use some form of hashing \citep{Shakhn_06a} to retrieve approximate neighbors quickly.

The out-of-sample prediction for a point in the training set does not generally equal the \z\ value it received during training (although it does not differ much from it either). That is, $\z(\x_n) \neq \z_n$, where $\z(\x_n)$ uses the training data $\w = \W_{n\cdot}$ and $\g = \G_{n\cdot}$ for $\x_n$. This is also true of semisupervised learning, and it simply reflects the fact that the out-of-sample mapping smoothes, rather than interpolates, the training data.

Given a solution \Z\ of the LASS training problem, the out-of-sample mapping is uniquely defined, because the problem~\eqref{e:lass-oos} is strongly convex. However, as described in section~\ref{s:unicity}, in particular settings the solution of the LASS training problem may not be unique, and a natural question is: what is the relation between the out-of-sample mappings for two different solutions? From th.~\ref{th:lass-sol}, the solutions have the form $\Z + \1_N \p^T$ where \Z\ is any particular solution and $\p \in \bbR^K$ satisfies $\p^T \1_K = 0$, $\p^T (\G^T \1_N) = 0$ and $\Z + \1_N \p^T \ge \0$. Then the out-of-sample mapping for a \p-solution has the form
\begin{equation*}
  \z_{\p}(\x) = \Pi\left( \frac{(\Z + \1_N \p^T)^T \w}{\1^T_N\w} + \gamma \g \right) = \Pi(\vv(\x) + \p) \qquad \vv(\x) = \frac{\Z^T \w}{\1^T_N\w} + \gamma \g
\end{equation*}
where $\z_{\0}(\x) = \Pi(\vv(\x))$ is the out-of-sample mapping for the base solution \Z. If \p\ was parallel to the vector $\1_K$ then the out-of-sample mappings for different solutions but actually coincide, but in fact $\p^T \1_K = 0$, so the out-of-sample mappings for different solutions correspond to sliding $\z_{\0}(\x)$ along the simplex by vector \p\ (which must respect the remaining conditions above, of course).

As a function of $\lambda$, the out-of-sample mapping takes the following extreme values:
\begin{itemize}
\item If $\lambda = 0$ or $\w = \0$, $z_{k} = \delta(k-k_{\text{max}})$ where $k_{\text{max}} = \arg\max\{g_{k},\ k=1,\dots,K\}$, i.e., the item is assigned to its most similar similar category (or any mixture thereof in case of ties).
\item If $\lambda = \infty$ or $\g = \0$, $\z = \bar{\z}$, independently of \g. This corresponds to the SSL out-of-sample mapping.
\end{itemize}
In between these, the out-of-sample mapping as a function of $\lambda$ is a piecewise linear path in the simplex, which represents the tradeoff between the crowd (\w) and expert (\g) wisdoms. This path is quite different from the simple average of $\bar{\z}$ and \g\ (which need not even be feasible), and may produce exact 0s or 1s for some entries.

The LASS out-of-sample mapping offers an extra degree of flexibility to the user, which may be used on a case-by-case basis for each test item. The user has the prerogative to set $\lambda$ to favor more or less the expert vs the crowd opinion, and in fact to explore the entire continuum for $\lambda\in[0,\infty)$. The user can also explore what-if scenarios by changing \g\ itself, given the vector \w\ (e.g.\ how would the assignment vector look like if we think that test item \x\ belongs to category $k$ but not to category $k'$?). These computations are all relatively efficient because the bottleneck, which is the $\calO(N)$ computation of $\bar{\z}$, is done once only.

Note that the out-of-sample mapping is nonlinear and nonparametric, and it maps an input \x\ (given its affinity information) onto a valid assignment vector in the probability simplex. Hence, LASS can also be considered as \emph{learning nonparametric conditional distributions over the categories, given partial supervision}.

\section{Related work}

\subsection{Semisupervised learning with a Laplacian penalty (SSL)}

In semisupervised learning (SSL) with a Laplacian penalty \citep{Zhu_03a}, the basic idea is that we are given an affinity matrix \W\ and corresponding graph Laplacian $\LL = \D - \W$ on $N$ items, and the labels for a subset of the items. Then, the labels for the remaining, unlabeled items are such that they minimize the Laplacian penalty, or equivalently they are the smoothest function on the graph that satisfies the given labels (``harmonic'' function). Call $\Z_u$ of $N_u \times K$ and $\Z_l$ of $N_l \times K$ the matrices of labels for the unlabeled and labeled items, respectively, where $N = N_l + N_u$, and $\Z^T = (\Z^T_u\ \Z^T_l)$. To obtain $\Z_u$ we minimize $\trace{\Z^T\LL\Z}$ over $\Z_u$, with fixed $\Z_l$:
\begin{multline}
  \label{e:ssl}
  \min_{\Z_u}{ \trace{\Z^T\LL\Z} } = \min_{\Z_u}{ \trace{
      \begin{pmatrix}
        \Z_u \\ \Z_l
      \end{pmatrix}^T
      \begin{pmatrix}
        \LL_u & \LL_{ul} \\ \LL^T_{ul} & \LL_l
      \end{pmatrix}
      \begin{pmatrix}
        \Z_u \\ \Z_l
      \end{pmatrix}
    } } = \min_{\Z_u}{ \trace{ \Z^T_u \LL_u \Z_u + 2 \Z^T_l \LL^T_{ul} \Z_u } + \text{constant} } \\
  \Rightarrow \Z_u = - \LL^{-1}_u \LL_{ul} \Z_l = \LL^{-1}_u \W_{ul} \Z_l.
\end{multline}
Thus, computationally the solution involves a sparse linear system of $N_u \times N_u$. An out-of-sample mapping for a new test item \x\ with affinity vector \w\ wrt the the training set can be derived by SSL again, taking $\Z_l$ of $N \times K$ as all the trained labels (given and predicted) and $\Z_u = \z^T$ as the free label. This gives a closed-form expression
\begin{equation}
  \label{e:ssl-oos}
  \z(\x) = \sum^N_{n=1}{\frac{w_n}{\sum^N_{n'=1}{w_{n'}}} \z_n}
\end{equation}
which is the average of the labels of \x's neighbors, making clear the smoothing behavior of the Laplacian. SSL with a Laplacian penalty is very effective in problems where there are very few labels, i.e., $N_u \ll N_l$, but the graph structure is highly predictive of each item's labels. Essentially, the given labels are propagated throughout the graph.

In our setting, the labels are the item-category assignments $z_{nk}$, and we have the following result.
\begin{thm}
  \label{th:ssl-assignments}
  In problem~\eqref{e:ssl}, if $\Z_l \ge \0$ and $\Z_l \1_{N_l} = \1_K$ then $\Z_u \ge \0$ and $\Z_u \1_{N_u} = \1_K$.
\end{thm}
\begin{proof}
  Since $\LL = \D - \W$ we have
  \begin{equation*}
    \0 = \LL \1_{N_u+N_l} =
    \begin{pmatrix}
      \LL_u & \LL_{ul} \\ \LL^T_{ul} & \LL_l
    \end{pmatrix}
      \begin{pmatrix}
        \1_{N_u} \\ \1_{N_l}
      \end{pmatrix} =
    \begin{pmatrix}
      \LL_u \1_{N_u} + \LL_{ul} \1_{N_l} \\ \LL^T_{ul} \1_{N_u} + \LL_l \1_{N_l}
    \end{pmatrix}
    \Rightarrow - \LL^{-1}_u \LL_{ul} \1_{N_l} = \1_{N_u}.
  \end{equation*}
  Hence $\Z_u \1_K = - \LL^{-1}_u \LL_{ul} \Z_l \1_K = - \LL^{-1}_u \LL_{ul} \1_{N_l} = \1_{N_u}$. That $\Z_u \ge \0$ follows from the maximum principle for harmonic functions \citep{DoyleSnell84a}: each of the unknowns must lie between the minimum and maximum label values, i.e., in $[0,1]$. (Strictly, they will lie in $(0,1)$ or be all equal to a constant.)
\end{proof}
Thus, in the special case where the given labels are valid assignments (nonnegative with unit sum), the predicted labels will also be valid assignments, and we need not subject the problem explicitly to simplex constraints, which simplifies it computationally. This occurs in the standard semisupervised classification setting where each item belongs to only one category and we use the $\z_n$ vectors to implement a $1$-of-$K$ coding (e.g.\ as used for supervised clustering in \citealp{Grady06a}). However, \emph{in general SSL does not produce valid assignments}, e.g.\ if the given labels are not valid assignments, or in other widely used variations of SSL, such as using class mass normalization \citep{Zhu_03a}, or using the normalized graph Laplacian instead of the unnormalized one, or using label penalties \citep{Zhou_04a}. In the latter case (also similar to the ``dongle'' variation of SSL; \citealp{Zhu_03a}), one minimizes the Laplacian penalty plus a term equal to the squared distance of the labeled points (considered free parameters as well) to the labels $\Z_l$ provided. Thus, this penalizes the labeled points from deviating from their intended labels, rather than forcing them to equal them. This was extended by \citet{SubramBilmes11a} (replacing squared losses with Kullback-Leibler divergences and adding an additional entropy term) to learning probability distributions, i.e., where the labels $\Z_l$ are entire distributions over the $K$ classes, with each item-class probability specified exactly. All these approaches rely on the following: \emph{they use provided, specific label values $\Z_l$ as targets to be (ideally) met by the parameters}.

\paragraph{Relation with LASS}

LASS and SSL are similar in that (1) \LL\ plays the same role, i.e., to propagate label information in a smooth way according to the item-item graph; and (2) both rely on some given data to learn \Z: the similarity matrix \G\ in LASS and the given labels $\Z_l$ in SSL. LASS and SSL differ as follows. (1) The use of the simplex constraints, necessary to ensure valid assignments, which also means all the assignment values in LASS are interdependent, unlike in the classical SSL, where the prediction for each category can be solved independently. (2) A fundamental difference is in the supervision provided. If in LASS we were given actual labels $\Z_l$ for some of the items, we would simply use them just as in SSL, and the LASS problem with $\Z_l$ having given assignments would be:
\begin{subequations}
  \label{e:lass-ssl}
  \begin{align}
    \label{e:lass-ssl-objfcn}
    \min_{\Z_u} & \quad \lambda \trace{ \Z^T_u \LL_u \Z_u + 2 \Z^T_l \LL^T_{ul} \Z_u } - \trace{\G^T_u \Z_u} \\
    \label{e:lass-ssl-c1}
    \text{s.t.} & \quad \Z_u \1_K = \1_{N_u} \\
    \label{e:lass-ssl-c2}
    & \quad \Z_u \ge \0.
  \end{align}
\end{subequations}
It is clear that setting $\G = \0$ there gives the SSL problem~\eqref{e:ssl} if $\Z_l$ contains valid assignments (so the constraints are redundant, from th.~\ref{th:ssl-assignments}). Thus, SSL is a particular case of LASS, not a different way of encoding the same input data.
However, the \G\ term provides soft, partial ``labels'', and this information differs from (hard) labels $\Z_l$. Indeed, when the label to be learned for each item is an assignment vector, the concept of ``labeling'' breaks down, for two reasons. First, if the number of categories $K$ is not very small and an item $n$ has nonzero assignments to multiple categories, in practice it is hard for a user to have to give a value (or tag) for every single relevant category. Giving partial information is much easier, by simply setting $g_{nk} = 1$ for the most relevant categories, possibly setting $g_{nk} = -1$ for a few categories, and setting $g_{nk} = 0$ for the rest (we stick to $\pm 1$ and $0$ similarities for simplicity). Second, because the assignment values are constrained to be in the simplex, \emph{we cannot give actual values for individual entries} (unless we give the entire assignment vector). For example, setting an entry to 1 implicitly forces the other entries to 0. In summary, \emph{the semantics of the item-category similarities in LASS is that, where nonzero, they encourage the corresponding assignment towards relatively high or low values (for positive and negative similarities, respectively), and where zero, they reflect ignorance and are non-committing}, something which is close to a user's intuition, but generally difficult to achieve by setting assignment values directly.

Not being able to commit to specific assignment values, especially where $g_{nk} = 0$, also implies that it is not possible to transform meaningfully a given item-category sparse similarity matrix \G\ into an assignment vector. Given a sparse matrix \G, if we insist in setting full assignments for each item $n$ having a nonzero vector $\g_n$ (so we can use these with SSL), perhaps the best one can do is to follow this labeling procedure: for each $k$, set $z_{nk} = 1$, $\epsilon$ or $0$ if $g_{nk} = 1$, $0$ or $-1$, respectively, and normalize $\z_n$ (where $1 \gg \epsilon \ge 0$ is a smoothing user parameter). This forces a zero assignment for each negative-similarity category, and distributes the unit assignment mass over the remaining categories. Obviously, this likely forces many $z_{nk}$ to wrong values and, as we show in the experiments, works poorly. (Another approach would be to set only the $z_{nk}$ entries for which $g_{nk} \neq 0$ and leave the others free (subject to the constraints~\eqref{e:lass-c1}--\eqref{e:lass-c2}), but this is essentially with LASS does.)

The difference between SSL and LASS is clearly seen in the out-of-sample mapping. In SSL, the information provided for a test item is just the vector \w\ of similarities to other items, and the SSL out-of-sample mapping coincides with $\bar{\z}$ in the LASS out-of-sample mapping, i.e., the average of its neighbors' assignments. With LASS, in addition to \w\ we also give the vector \g\ of similarities to categories. If $\g=\0$, the predictions of LASS and SSL coincide. Otherwise, LASS trades off both \w\ and \g. This is particularly important when \w\ is not very informative, e.g.\ if it has many nonzero entries of similar magnitude, or all entries are very small (an outlying or ``new'' item).

A further difference between SSL and LASS is that in SSL the learned assignments $z_{nk}$ are never exactly $0$ or $1$ (in nontrivial problems) because each $\z_n$ is the average of its neighbors' assignments. In LASS, the $z_{nk}$ values can be exactly $0$ or $1$. This happens for the same reason why in many statistical models the $L_1$ norm achieves sparse solutions while the $L_2$ norm does not (indeed, the simplex constraints~\eqref{e:lass-c1}--\eqref{e:lass-c2} imply $\smash{\norm{\z_n}}_1 = 1$ for each $n$).

\subsection{Assignments and probabilities}

The semantics of assignments is different from that of probabilities. Given a discrete probability distribution $p$ over the values in $\Omega = \{1,\dots,K\}$, saying that ``$p(Z = k) = z_{k}$'' means that $Z$ equals exactly one of the values in $\Omega$, but that we observe the value $k$ with a probability $z_{k}$ (frequentist or Bayesian), for each $k \in \Omega$. In contrast, given an assignment vector \z\ over the values in $\Omega$, the value of $z_k$ corresponds to the proportion of \z\ that belongs to category $k$. For example, in a portfolio model, $k=1,\dots,K$ are possible investments and $z_k$ is the portion of a \$1 capital that is allocated to investment $k$, not the probability that we allocate all the \$1 to investment $k$. However, in some cases the assignments can indeed be used as proxies for probabilities. For example, our model includes classification as a particular (and rather restricted) setting, where (most) items belong to a single category and categories are mutually exclusive. Here, the assignment vectors may be interpreted as probabilities and the labels safely provided as $1$--of--$K$ codes.

Assignment problems have a long history in operations research and economics. For example, the Markowitz portfolio model \citep{NocedalWright06a} seeks the portfolio (soft assignment of an individual investor's \$1 capital to $K$ investments) that maximizes the expected return and minimizes the variance, and has the form of a QP. However, the use of Laplacian penalties (which, in the Markowitz model, would describe the similarity between different investors) does not seem to have been applied there.

\subsection{Other applications of LASS}

The LASS problem~\eqref{e:lass} appears in the assignment step of the training algorithm for Laplacian $K$-modes clustering \citep{WangCarreir14c}. Here, the items are data points $\x_1,\dots,\x_N \in \bbR^D$ to be clustered into $K$ clusters and the categories are the cluster centroids $\c_1,\dots,\c_K \in \bbR^D$. Both affinities are Gaussian: the \W\ matrix has entries $w_{nm} = G(\norm{(\x_n-\x_m)/\sigma}^2)$, $n,m=1,\dots,N$, and the \G\ matrix has entries $g_{nk} = G(\norm{(\x_n-\c_k)/\sigma}^2)$, $n=1,\dots,N$, $k=1,\dots,K$, where $G(\cdot^2)$ gives a Gaussian kernel. At an optimum Laplacian $K$-modes clustering, the cluster centroids $\C= (\c_1,\dots,\c_K)$ are the modes of the kernel density estimates defined by each of the $K$ clusters' points. The problem is optimized over \Z\ and \C\ in alternating steps, where the step over \Z\ for fixed \C\ has the form~\eqref{e:lass}.

Supervised clustering can be formulated as problem~\eqref{e:lass}, in particular supervised image segmentation. The user marks several pixels for each segment of the image, from which we can define the similarity matrix \G\ in various ways. The similarity matrix \W\ is constructed from the image (pixel location and range features).

\section{Experiments}
\label{s:expts}

\begin{figure}[t]
  \centering
  \begin{tabular}{cc}
    \psfrag{C1}{C1}
    \psfrag{C2}{C2}
    \psfrag{C3}{C3}
    \psfrag{C4}{C4}
    \psfrag{a}{$a$}
    \psfrag{b}{$b$}
    \psfrag{c}{$c$}
    \psfrag{d}{$d$}
    \psfrag{e}{$e$}
    \psfrag{f}{$f$}
    \psfrag{g}{$g$}
    \psfrag{a'}{$a'$}
    \psfrag{b'}{$b'$}
    \psfrag{c'}{$c'$}
    \psfrag{d'}{$d'$}
    \psfrag{e'}{$e'$}
    \psfrag{f'}{$f'$}
    \psfrag{g'}{$g'$}
    \raisebox{-14ex}{\includegraphics[width=0.45\linewidth]{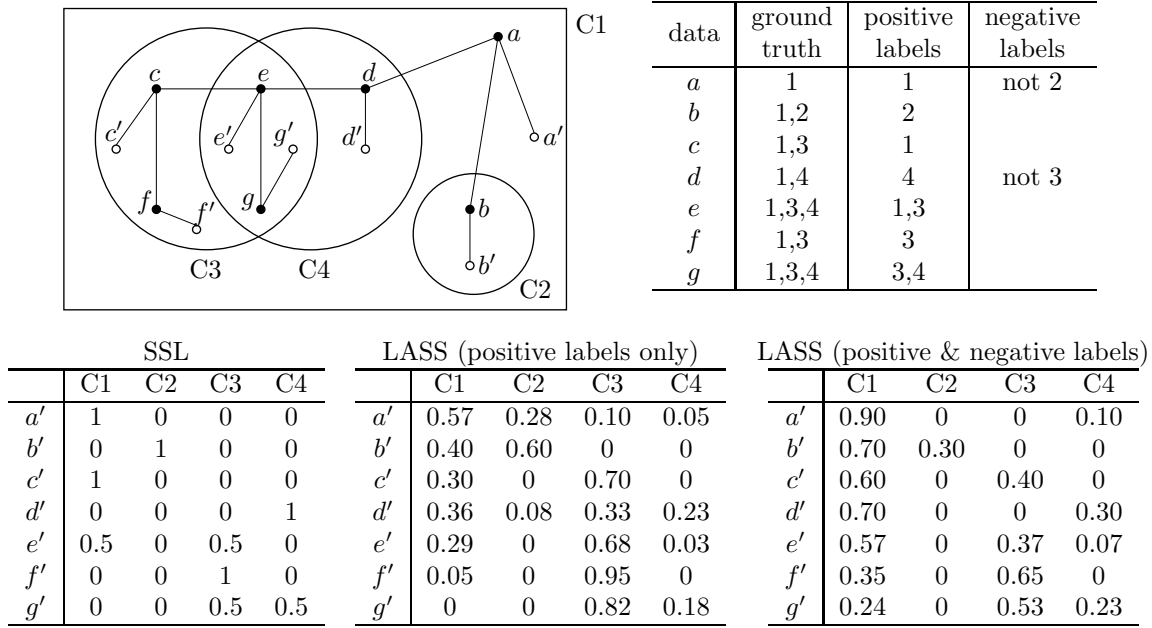}} &
    \begin{tabular}{c|c|c|c}
      \hline
      data & \caja{c}{c}{ground \\ truth} & \caja{c}{c}{positive \\ labels} & \caja{c}{c}{negative \\ labels} \\
      \hline
      $a$ & 1     & 1   & \mbox{not} 2 \\
      $b$ & 1,2   & 2   & \\
      $c$ & 1,3   & 1   & \\
      $d$ & 1,4   & 4   & \mbox{not} 3 \\ 
      $e$ & 1,3,4 & 1,3 & \\
      $f$ & 1,3   & 3   & \\
      $g$ & 1,3,4 & 3,4 & \\
      \hline
    \end{tabular}
  \end{tabular} \\[2ex]
  \begin{tabular}{@{}ccc@{}}
    SSL & LASS (positive labels only) & LASS (positive \& negative labels) \\
    \begin{tabular}{c|cccc}
      \hline
      & C1  & C2 &C3 &C4 \\
      \hline
      $a'$ & 1  & 0  & 0  & 0 \\
      $b'$ & 0  & 1  & 0  & 0 \\
      $c'$ & 1  & 0  & 0  & 0 \\
      $d'$ & 0   & 0 & 0  & 1 \\
      $e'$ & 0.5 & 0 & 0.5& 0 \\
      $f'$ & 0   & 0 & 1  & 0 \\
      $g'$ & 0   & 0 & 0.5& 0.5 \\
      \hline
    \end{tabular}&
    \begin{tabular}{c|cccc}
      \hline
      & C1  & C2 &C3 &C4 \\
      \hline
      $a'$ & 0.57& 0.28  & 0.10  & 0.05 \\
      $b'$ & 0.40  & 0.60  & 0  & 0 \\
      $c'$ & 0.30  & 0  & 0.70  & 0 \\
      $d'$ & 0.36  & 0.08 & 0.33  & 0.23 \\
      $e'$ & 0.29 & 0 & 0.68& 0.03 \\
      $f'$ & 0.05   & 0 & 0.95  & 0 \\
      $g'$ & 0   & 0 & 0.82& 0.18 \\
      \hline
    \end{tabular}&
    \begin{tabular}{c|cccc}
      \hline
      & C1  & C2 &C3 &C4 \\
      \hline
      $a'$ & 0.90  & 0  & 0  & 0.10 \\
      $b'$ & 0.70  & 0.30  & 0  & 0 \\
      $c'$ & 0.60  & 0  & 0.40  & 0 \\
      $d'$ & 0.70   & 0 & 0  & 0.30 \\
      $e'$ & 0.57 & 0 & 0.37 & 0.07 \\
      $f'$ & 0.35   & 0 & 0.65  & 0 \\
      $g'$ & 0.24   & 0 & 0.53 & 0.23 \\
      \hline
    \end{tabular}
  \end{tabular}
  \caption{Illustrative example. \emph{Top}: Venn diagram of the categories and ground-truth labels for the labeled items. \emph{Bottom}: assignments found by SSL and LASS (using only positive labels, and both positive and negative labels).}
  \label{f:toy}
\end{figure}

\paragraph{Illustrative example}

We constructed a simple example to show the difference between LASS and SSL, and the role of positive and negative item-category similarities. The data consists of $N = 14$ points $a$--$g$ and $a'$--$g'$ and $K = 4$ categories C1--C4 which are related to each other according to the Venn diagram in fig.~\ref{f:toy}. Category C1 contains the other three categories and C3 and C4 intersect. This is quite different from the usual classification setting where each point is assigned to only one category. A graph is built on the data set, denoted by the edges between data points. We consider $a$--$g$ (filled circles) as partially labeled and $a'$--$g'$ (hollow circles) as completely unlabeled. The ground-truth (true assignments derived from the Venn diagram), positive and negative labels we use are given in fig.~\ref{f:toy}(top right). For LASS, we set $g_{nk} = \pm 1$ if $\x_n$ is given positive/negative label C$k$, and set $g_{nk} = 0$ if no information is given. We do not give the full label information for any point (except $a$), so it is crucial to make use of the graph Laplacian to propagate label information. Since each unlabeled point $\x_i$ is only connected to one (partially) labeled point $\x_j$, $\x_i$ will inherit exactly the same assignment from $\x_j$ in order to minimize the objective function (in this case $\smash{w_{ij}\norm{\z_i-\z_j}^2} = 0$). Fig.~\ref{f:toy}(bottom) shows the assignment of unlabeled points obtained from SSL and from LASS ($\lambda=1$) with only positive labels, and both positive and negative labels in \G. With SSL \citep{Zhu_03a}, since the labels are used as constraints (eq.~\eqref{e:ssl}), the unlabeled points just inherit them and miss many assignments. With LASS, using only positive \G\ obtains smoother assignments, but it sometimes assigns points to categories they do not belong to (e.g.\ $a'$ to C2, $d'$ to C3). Once we use negative labels, the progagation of those wrong labels is cut off and we obtain assignments closer to the ground truth for all points. SSL cannot easily use the negative labels.

\paragraph{Digit recognition}

We test LASS in a classification task where each data point has only one valid label. We randomly sample $10\,000$ MNIST digit images and compute the $10$-nearest-neighbor graph with similarities based on the Euclidean distance between images. We then randomly select $N_l$ images from each of the 0--9 categories and give them the correct label. We compare with a nearest neighbor classifier (NN), one-vs-all kernel support vector machine (KSVM) using RBF kernel of width $\sigma=5$ and hinge loss penalty parameter $C$ selected from $\{10^{-3},10^{-2},10^{-1},1,10^1,10^2,10^3\}$ and SSL \citep{Zhu_03a,Grady06a} using 1-out-of-10 coding for the labeled points. We also include two variants of SSL: SSL1 \citep{Zhu_03a} normalizes the assignments from SSL so that the prior distribution of the different classes is respected. SSL2 uses the normalized graph Laplacian instead of the unnormalized one in SSL \citep{Zhou_04a}. SSL1 and SSL2 improve over SSL but neither of them produce valid assignments (they do not lie on the probability simplex). For LASS and SSL/1/2 we assign each point to the category with the highest prediction value. We let all algorithms use their respective optimal parameters (e.g.\ $\lambda$ in LASS is determined by a grid search).

Fig.~\ref{f:mnist-20news}(left) shows the classification error over $20$ different labeled/unlabeled partitions of the dataset as a function of $N_l$ (errorbars not shown to avoid clutter). The accuracy of all algorithms improves as the number of labeled points increase, particularly for NN and KSVM, which are template matchers. But when only few points are labeled, the methods that make use of Laplacian smoothing significantly outperform them. LASS (runtime: 40 s) consistently achieves the best accuracy while producing valid assignments.

\paragraph{Document categorization}

We want to predict assignments of documents to topics, where each document may belong to multiple topics, in the 20-newsgroups dataset ($N=11\,269$ documents). We manually add 7 new topics (comp.sys, rec.sport, computer, recreation, politics, science and religion) based on the hierarchical structure and perceived similarity of groups (e.g.\ comp.sys.ibm.pc.hardware / comp.sys.mac.hardware). This yields $K = 27$ topics and each document can belong to 1 to 3 topics. To construct feature vectors, we remove words that occur in 5 or fewer documents, and then extract the TFIDF (term frequency $\times$ inverse document frequency) feature of the documents. We generate the \G\ similarity matrix by randomly selecting $N_l$ documents from each of the $27$ topics, and giving each document one $+1$ label (the topic it is selected from) and five $-1$ labels (topics it does not belong to). For SSL, we turn this into assignment labels as described in section~\ref{e:ssl}, and select the smoothing parameter $\epsilon$ optimally from $\{0,\, 0.005,\,  0.01,\, 0.02,\, 0.05,\, 0.1,\, 0.2\}$.

\begin{figure}[t]
  \centering
  \psfrag{NN}{NN}
  \psfrag{SEMI}{SSL}
  \psfrag{SEMI-1}{SSL1}
  \psfrag{SEMI-2}{SSL2}
  \psfrag{KSVM}{KSVM}
  \psfrag{SVM}{SVM}
  \psfrag{Ours}{LASS}
  \psfrag{nl}[t][]{$N_l$}
  \begin{tabular}{@{}c@{\hspace{0.0\linewidth}}c@{}}
    MNIST & 20-newsgroups \\
    \psfrag{err}[][t]{classification error}
    \includegraphics[width=0.50\linewidth,height=0.375\linewidth]{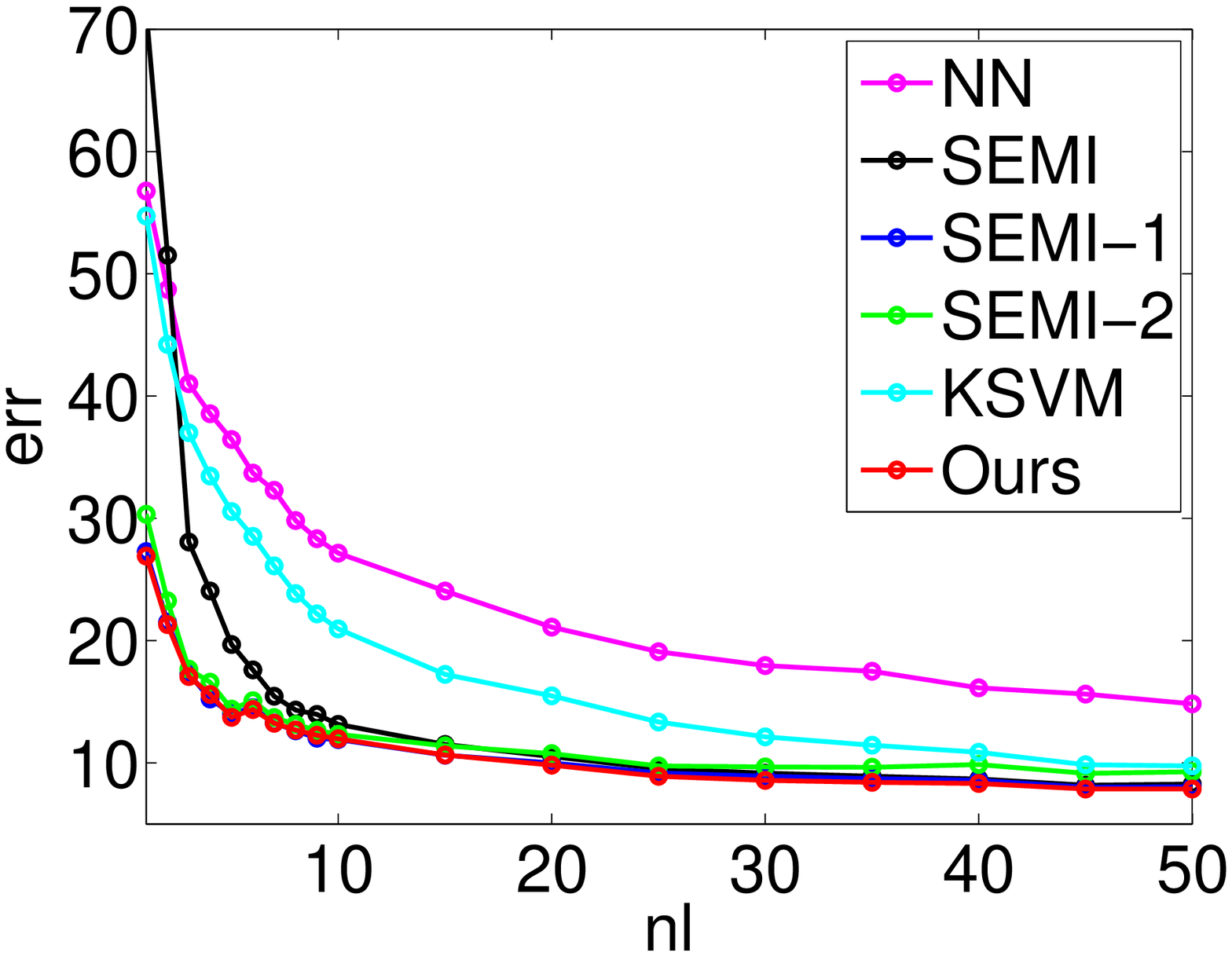} &
    \psfrag{err}[][]{}
    \includegraphics[width=0.50\linewidth]{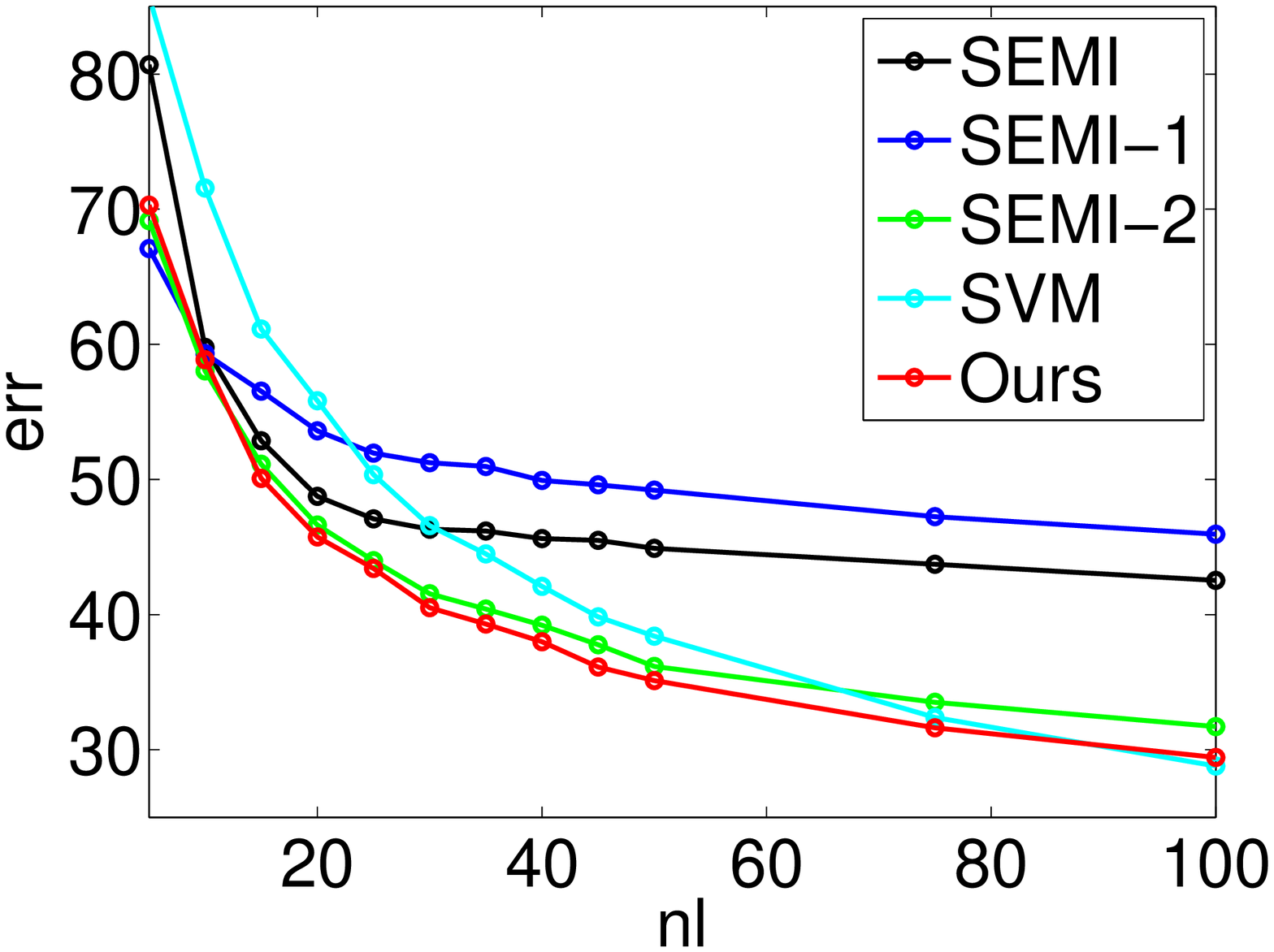}
  \end{tabular}
  \caption{Classification error (\%) vs number of labeled points for each class on MNIST (left) and 20-newsgroups (right) datasets.}
  \label{f:mnist-20news}
\end{figure}

To evaluate the performance, we select for each test document the $T$ topics to which it has highest predicted assignments (where $T\in\{1,2,3\}$ is the actual number of topics this document belongs to), consider them as predicted label set, and compare them with ground truth labels. We consider it as an error if the predicted label set and the ground truth set differ. NN classification does not apply here because no document is fully labeled. We can apply one-vs-all linear SVM (hinge loss penalty parameter $C$ selected from $\{10^{-3},10^{-2},10^{-1},1,10^1,10^2,10^3\}$) because we do have training points for each topic. Fig.~\ref{f:mnist-20news}(right) shows the mean classification error over $20$ random labeled/unlabeled partitions of the dataset as a function of $N_l$. The accuracy of all algorithms again improves when $N_l$ increases. LASS outperforms all other algorithms at nearly all $N_l$ values, and, unlike them, always produces valid assignments.

\paragraph{Image tagging}

This experiment fully benefits from the ability of LASS to handle partial labels, and predict full assignment vectors. In the problem of image tagging, each image can typically be tagged with multiple categories of a large number of possible categories. However, a user will usually tag only a few of the relevant categories for it (e.g.\ out of laziness) and will miss tagging other categories that would be relevant too. The task given a test image is to predict an assignment vector, i.e., to fill in ``soft tags'', for all categories. When there are many possible categories, trying to fill in the missing assignments for the partially labeled samples (so we can use SSL to propagate them to the unlabeled samples) is pointless. In contrast, LASS does not require these missing assignments, by conveniently providing zero affinities.

We demonstrate LASS on a subset of the ESP game \citep{AhnDabbis04a} images used by \citet{Guillaum_09a}. We select the images in the training set that are tagged with at least $6$ categories (words), resulting in $6\,100$ images with a total of $267$ non-empty categories, with $7.2$ categories per image on average. We use the same image feature sets as \citet{Guillaum_09a} to compute distances between images and build a $10$-nearest neighbor graph. We give partial information for $4\,600$ images and provide item-category affinities for each image in the following way: we give positive affinity ($+1$) for a random subset of size $n_l$ from the categories it is tagged with, and give negative affinities ($-1$) randomly for $5$ out of the $20$ most frequent categories it is not tagged with. Providing negative affinities in this way stops the algorithm from concentrating most of the probability mass on the most frequent categories. The other $1\,500$ images are completely unlabeled and used for testing. SSL fills in missing assignments of the partially labeled samples as described in the document categorization experiment. Parameters are selected based on grid search for each algorithm.

We evaluate the performance of different algorithms using the precision, recall and F--1 score (averaged over sample images) on the test samples while fixing the annotation length at $5$, i.e., each image is tagged with the $5$ categories of highest assignment. (Although LASS admits tags as similarity values at test time, we do not use them here.) We vary the number of positive tags $n_l$ from $1$ to $6$. Fig.~\ref{f:espgame} shows the results for SSL and LASS over $20$ runs (each with a different random selection of test set and partial affinities). We could not run one-versus-all SVMs because there are no negative samples for most categories. In SSL2, the highest prediction values are nearly always the most frequent categories. We see that LASS greatly improves over SSL, especially when smaller numbers of positive affinities are given.

\begin{figure}[th!]
  \centering 
  \psfrag{SEMI}{SSL}
  \psfrag{SEMI-1}{SSL1}
  \psfrag{SEMI-2}{SSL2}
  \psfrag{Ours}{LASS}
  \psfrag{nl}[t][]{$n_l$}
  \psfrag{precision}[][]{Precision (\%)}
  \psfrag{recall}[][]{Recall (\%)}
  \psfrag{F1}[][]{F--1}
  \begin{tabular}{@{}c@{\hspace{0.005\linewidth}}c@{\hspace{0.005\linewidth}}c@{}}
    Precision & Recall & F--1 score \\
    \includegraphics[width=0.33\linewidth]{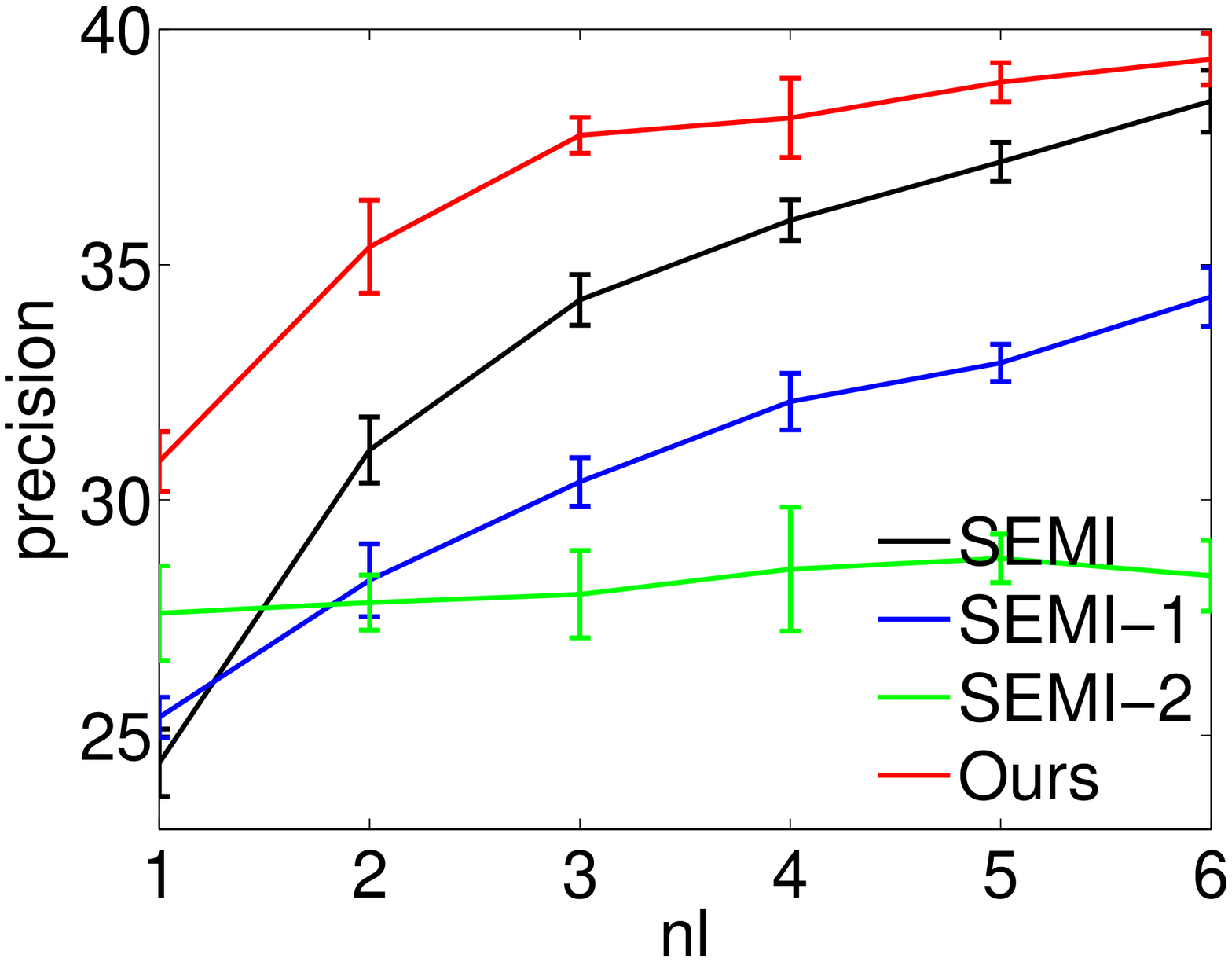} &
    \includegraphics[width=0.33\linewidth]{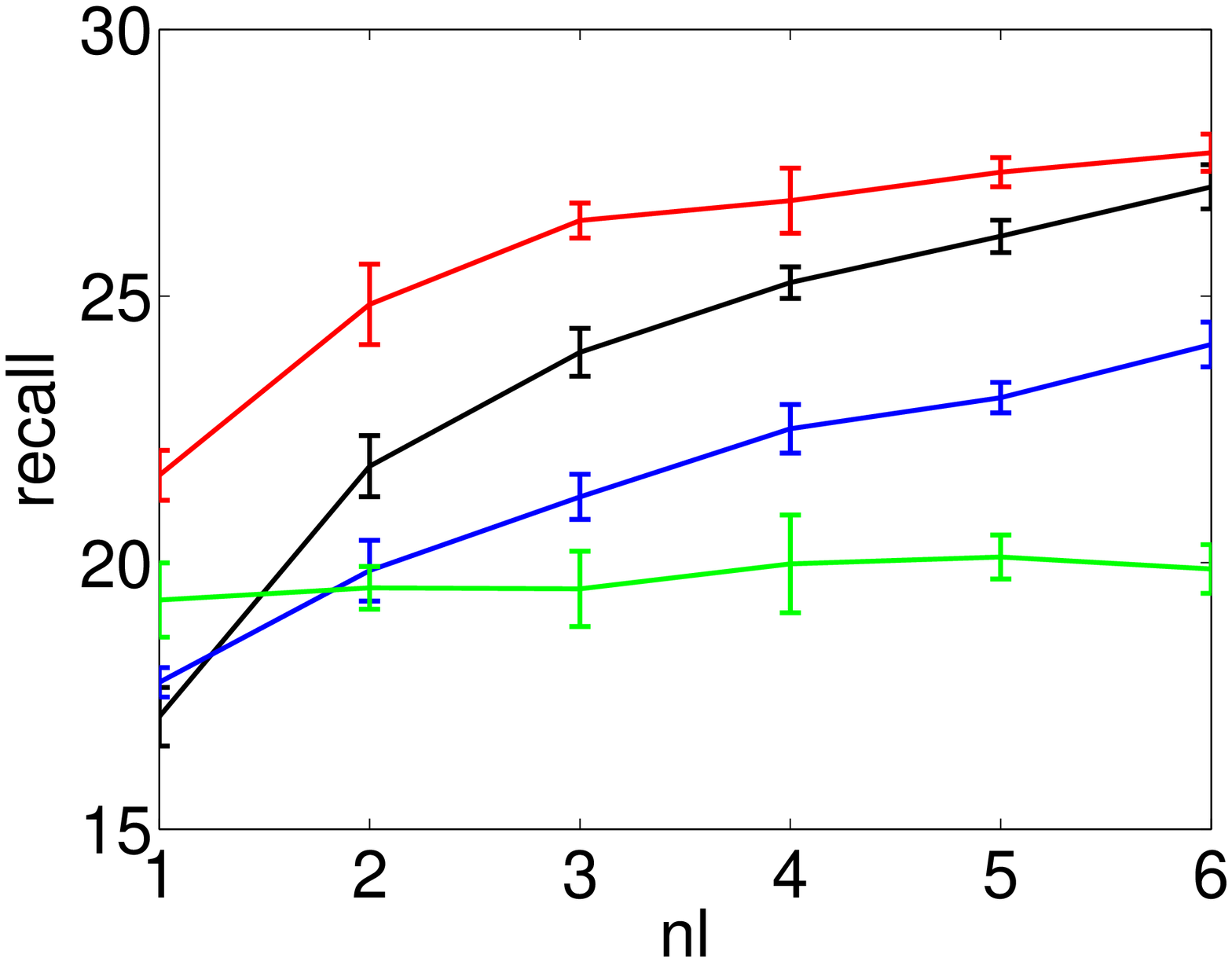} &
    \includegraphics[width=0.33\linewidth]{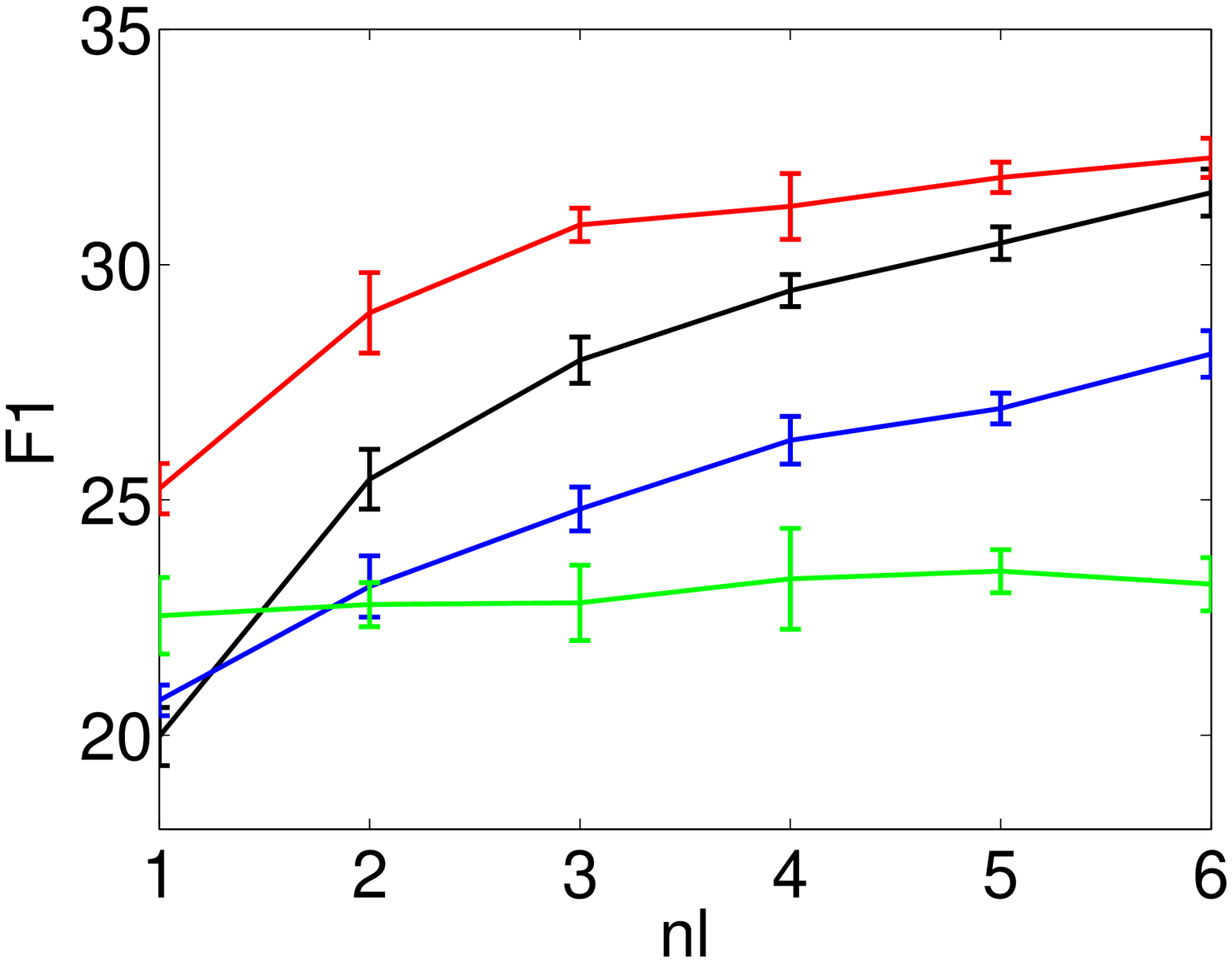}
  \end{tabular} \\[1ex]
  \begin{tabular}{@{}c@{\hspace{0.05\linewidth}}c@{\hspace{0.05\linewidth}}c@{}}
    \parbox[t]{0.3\linewidth}{
      \emph{GT}: dog grass green man sky white \\ 
      \emph{Pred.}: grass man sky green white \textbf{tree}} &
    \parbox[t]{0.3\linewidth}{
      \emph{GT}: black drawing man old soldier tent white \\
      \emph{Pred.}: black old drawing white tent \textbf{sketch} man} &
    \parbox[t]{0.3\linewidth}{
      \emph{GT}: blue computer gray purple screen window \\
      \emph{Pred.}: computer screen gray window  blue \textbf{white}} \\[8.5ex]
    \includegraphics[width=0.3\linewidth,height=0.19\linewidth]{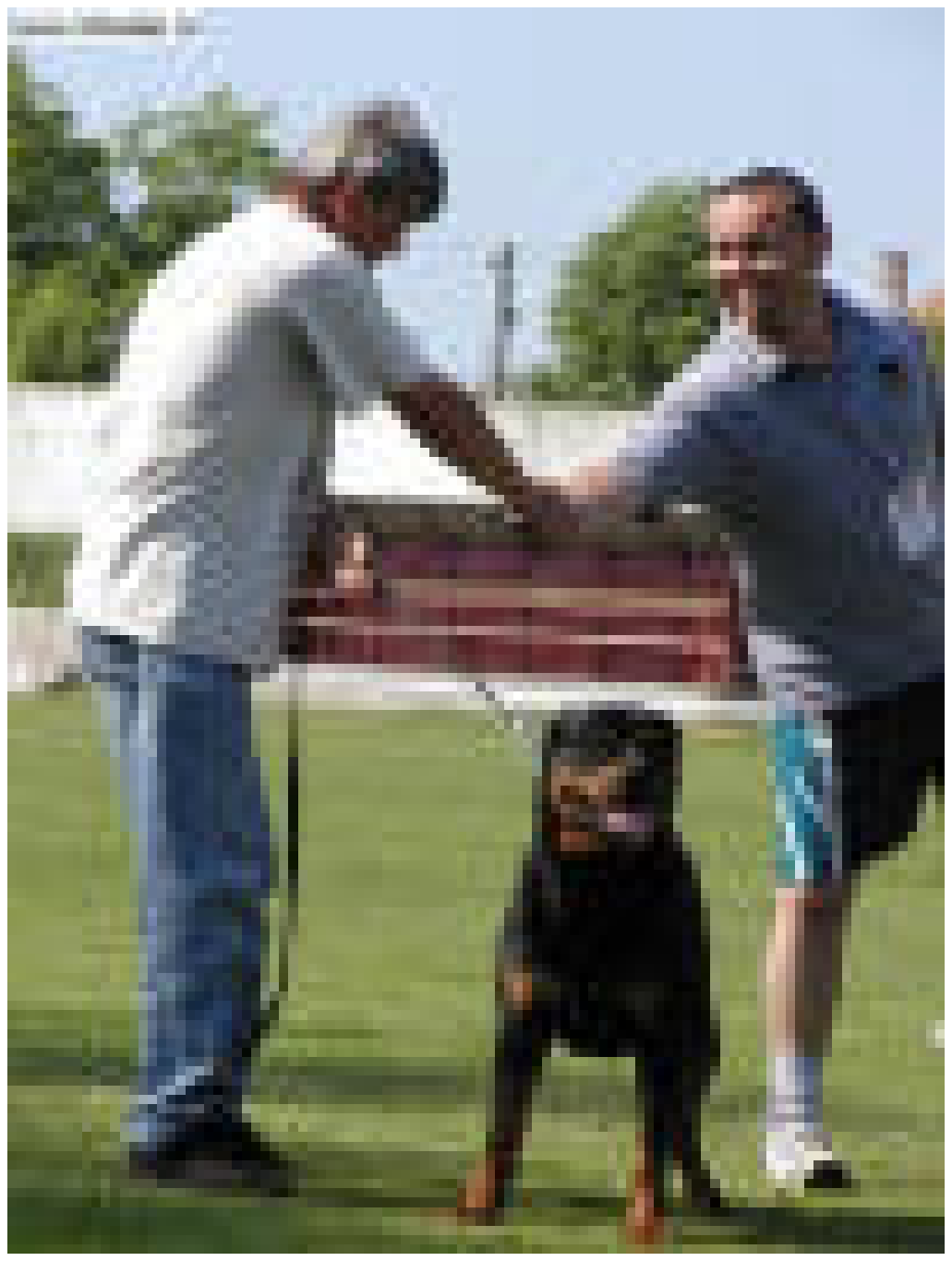} &
    \includegraphics[width=0.3\linewidth,height=0.19\linewidth]{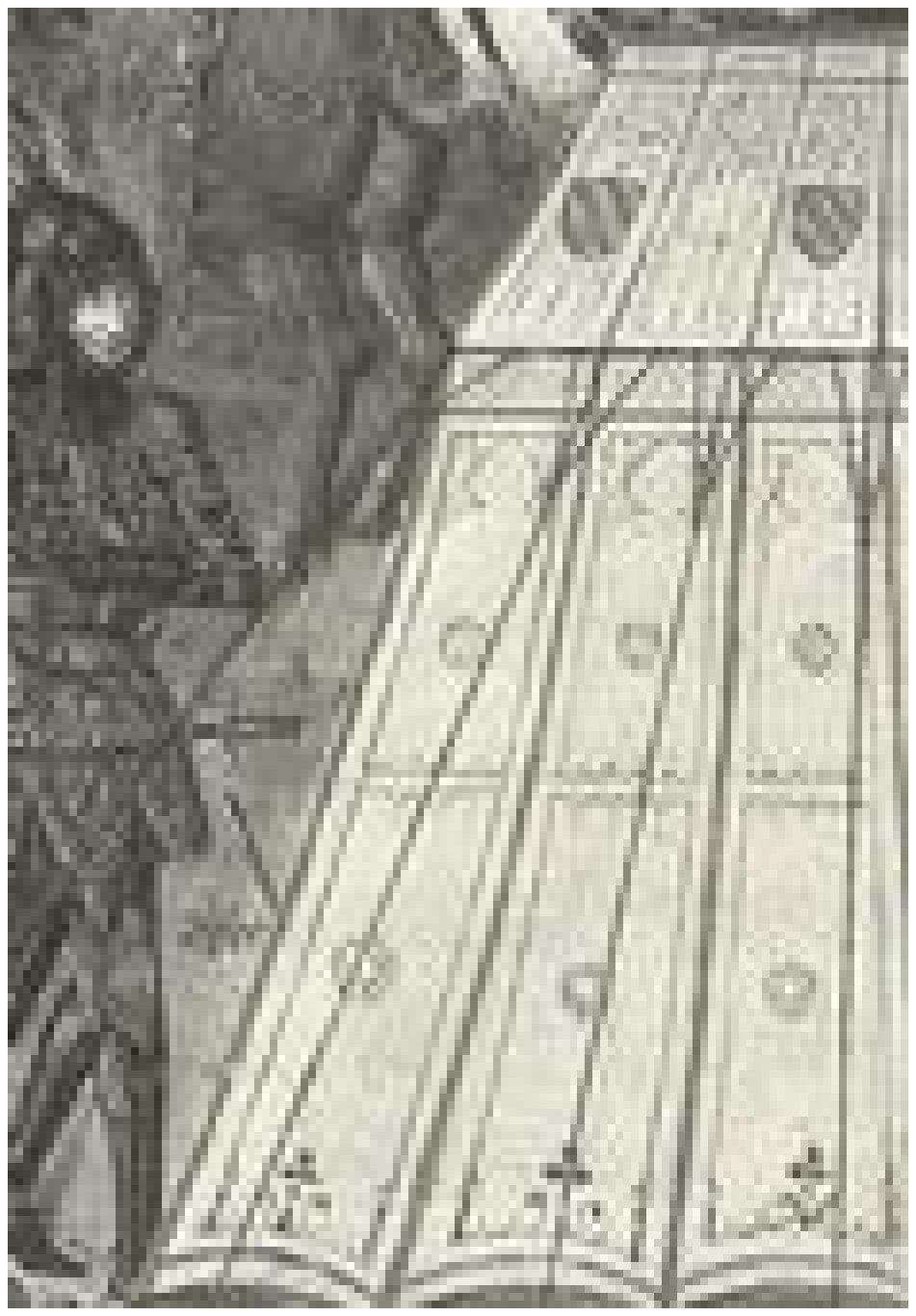} &
    \includegraphics[width=0.3\linewidth,height=0.19\linewidth]{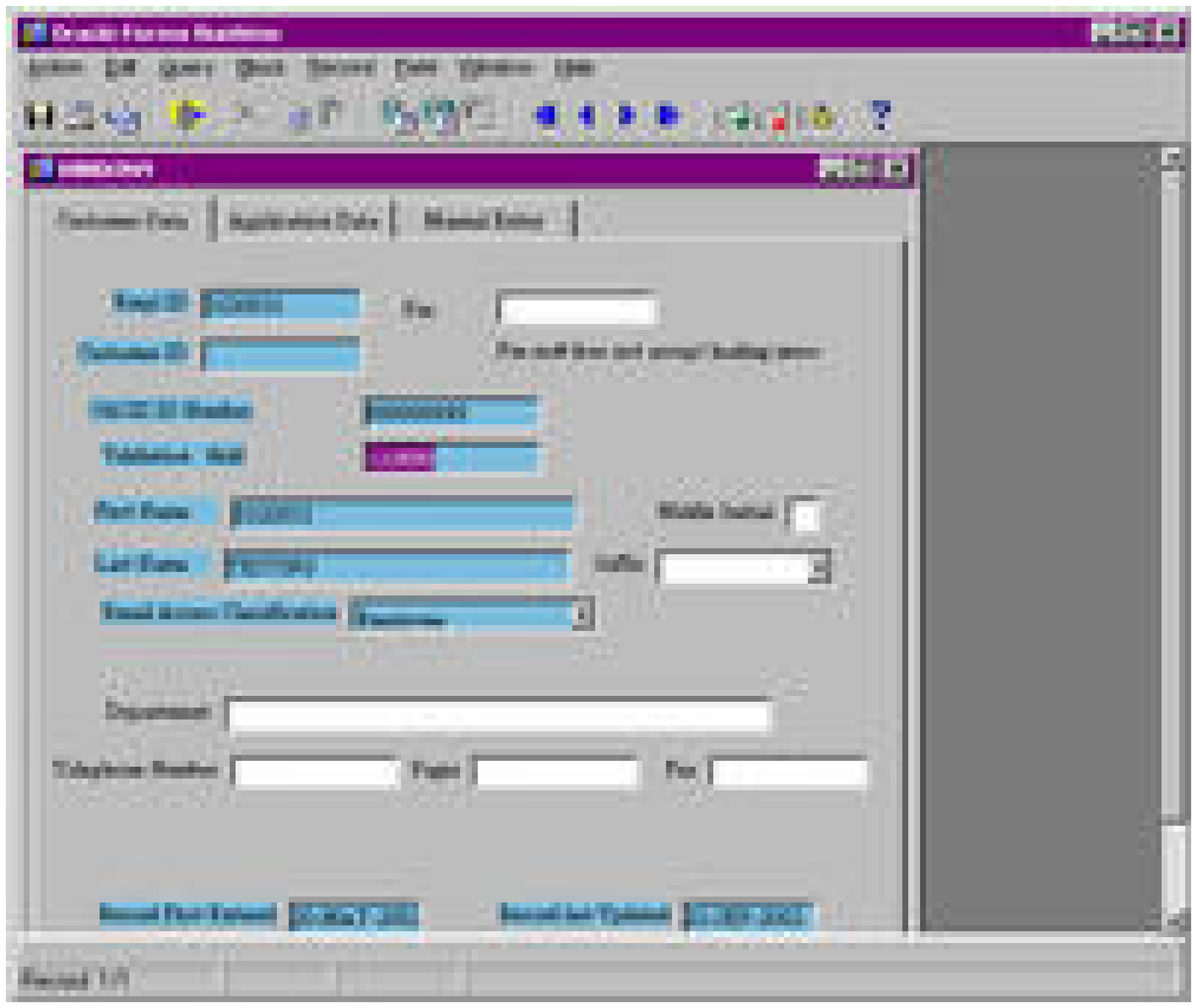} \\[2ex]
    \parbox[t]{0.3\linewidth}{
      \emph{GT}: black drawing hair man nose old white \\
      \emph{Pred.}: black  white drawing man hair \textbf{circle} \textbf{tie}}&
    \parbox[t]{0.3\linewidth}{
      \emph{GT}: black coin man money old round silver white \\
      \emph{Pred.}: black old round coin money  man \textbf{woman} \textbf{gray}} &
    \parbox[t]{0.3\linewidth}{
      \emph{GT}: field grass green people sky tree \\
      \emph{Pred.}: grass sky green \textbf{man} tree \textbf{tent}} \\[8.5ex]
    \includegraphics[width=0.3\linewidth,height=0.19\linewidth]{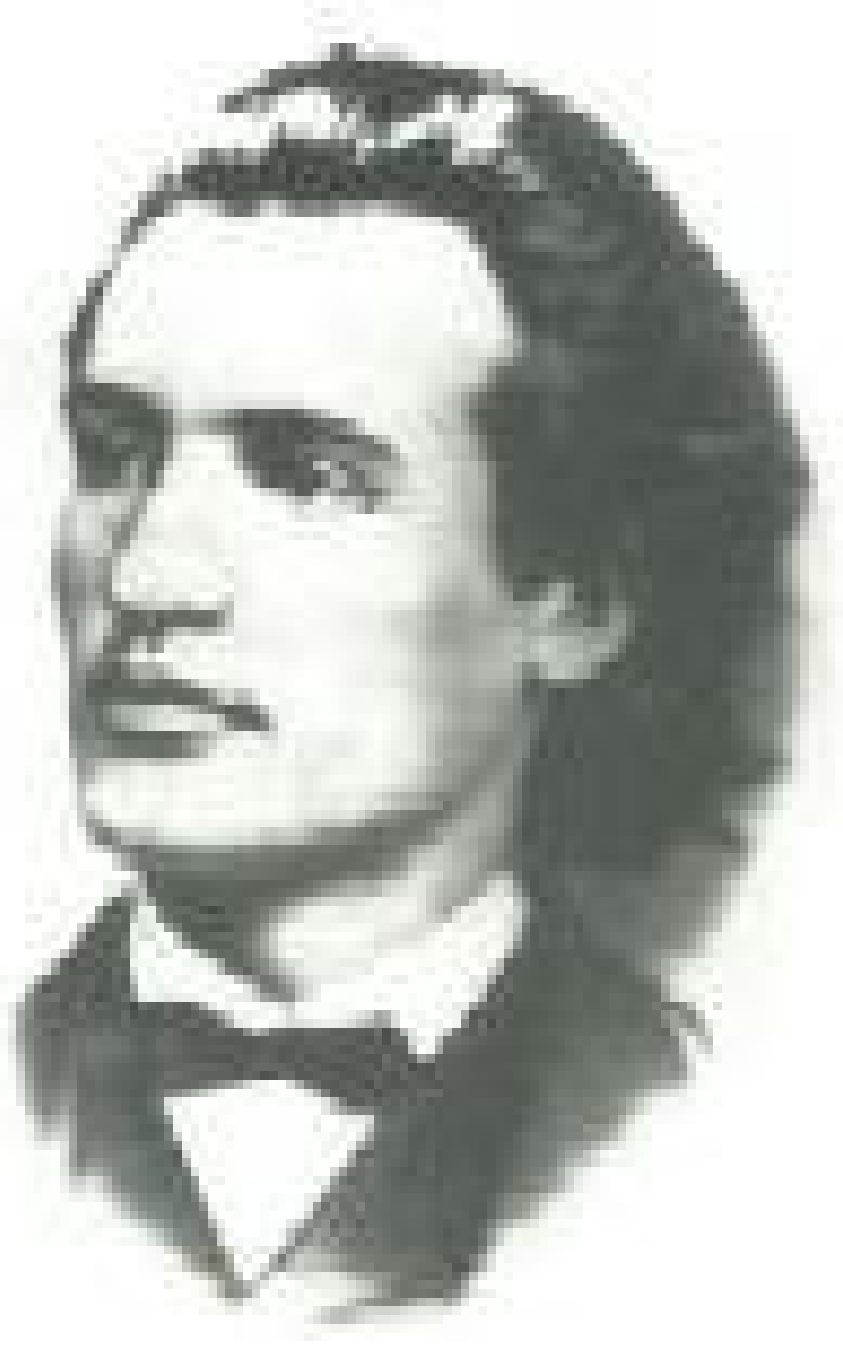} &
    \includegraphics[width=0.3\linewidth,height=0.19\linewidth]{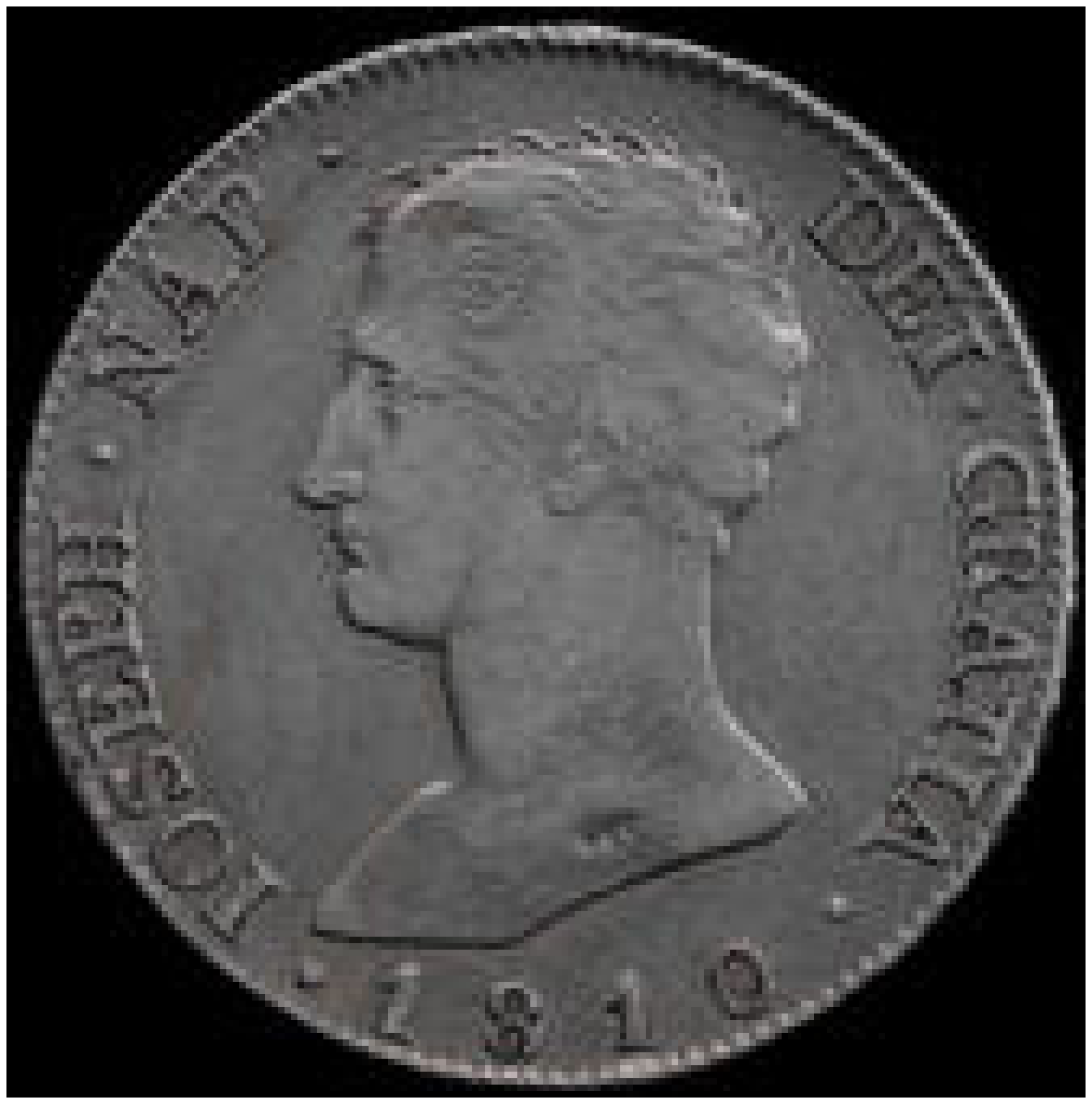} &
    \includegraphics[width=0.3\linewidth,height=0.19\linewidth]{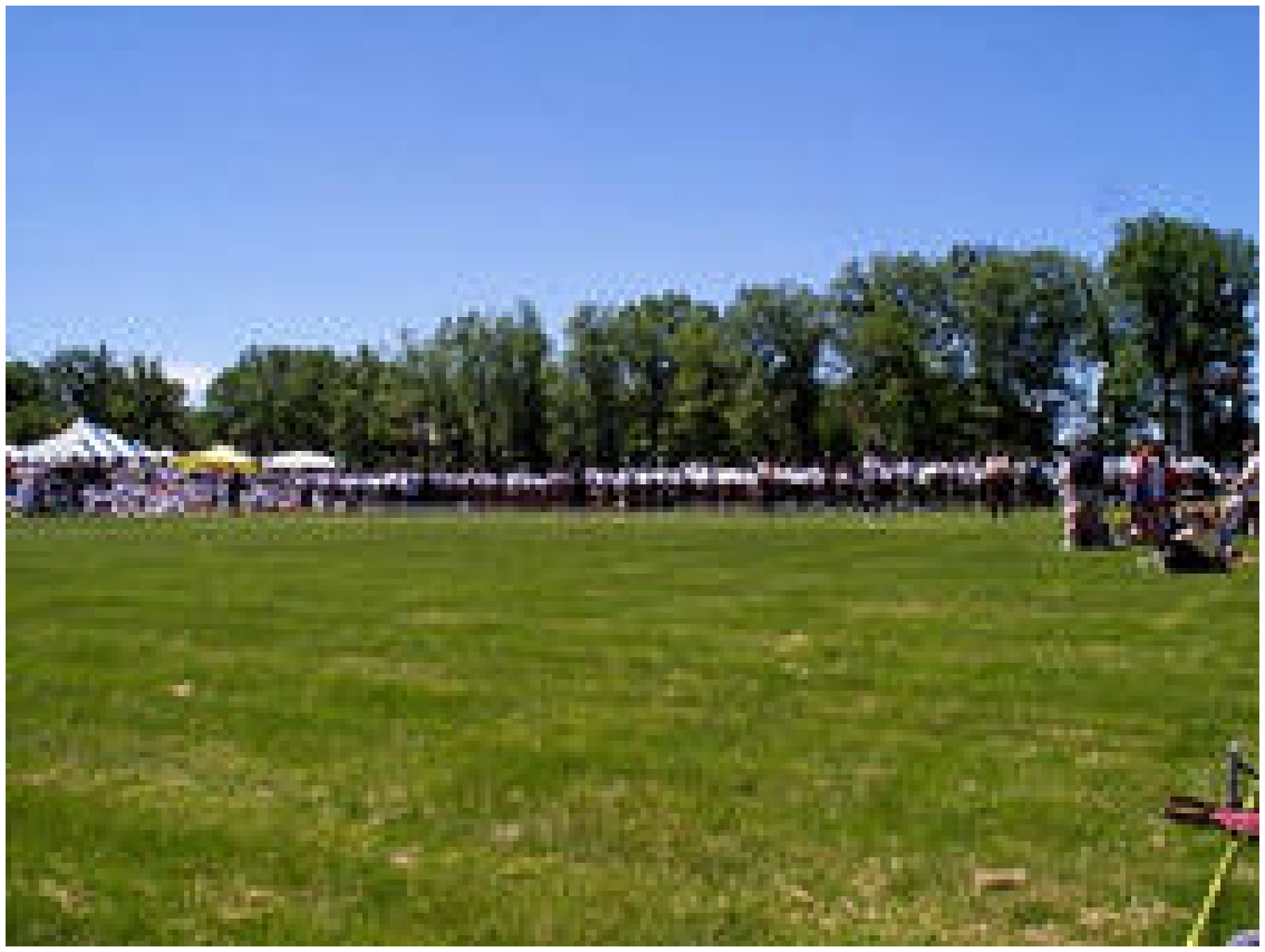}
  \end{tabular}
  \caption{Results on ESP game dataset. \emph{Top}: precision and recall vs number of positive affinities $n_l$, at a fixed annotation length $5$, averaged over $20$ runs. \emph{Bottom}: sample predictions of LASS on test images when $n_l=5$. We show the highest predictions (\emph{Pred.}, in assignment order) up to the number of tags in the ground truth (\emph{GT}, unordered), with \textbf{mismatches} in boldface.}
  \label{f:espgame}
\end{figure}

\section{Conclusion}

We have proposed a simple quadratic programming model for learning assignments of items to categories that combines two complementary and possibly conflicting sources of information: the crowd wisdom and the expert wisdom. It is particularly attractive when fully labeling an item is impractical, or when categories have a complex structure and items can genuinely belong to multiple categories to different extents. It provides a different way to incorporate supervision to that of traditional Laplacian semisupervised learning, which is ill-suited for this setting because the similarity information cannot be faithfully transformed into assignment labels.

We have derived a training algorithm based on the alternating direction method of multipliers. The algorithm's iterations can be made fast by caching the Cholesky factorization of the graph Laplacian. The algorithm is very simple to implement. It requires no line searches and has only one user parameter, the penalty parameter. The algorithm converges for any positive value of the penalty parameter, but this value does affect the convergence rate.

We expect LASS to apply to problems beyond semisupervised learning, such as clustering, and in social network applications, with image, sound or text data that is partially tagged. It can also be extended to handle tensor data or have additional terms in its objective, for example to represent relations between categories with a category-category similarity matrix, or even to use negative similarities in the item-item graph (since the feasible set is bounded, solution(s) still exist). A further application of LASS is to learn probability distributions that are conditional on partial supervisory information, since, in effect, the out-of-sample mapping is a nonparametric mapping from the affinity information to a distribution over categories. Another research direction is to find optimally adaptive schedules for the penalty parameter and to accelerate the convergence of the training algorithm, particularly with large datasets with many categories, where we may expect each row of \Z\ to be sparse.


\end{document}